\pdfoutput=1
\documentclass[10pt, conference, letterpaper]{IEEEtran}
\IEEEoverridecommandlockouts
\usepackage{cite}
\usepackage{amsmath,amssymb,amsfonts}
\usepackage{graphicx}
\usepackage{textcomp}
\usepackage{algorithm, algorithmic, setspace}
\usepackage{xcolor}
\usepackage{amsthm}
\usepackage{hyperref}
\newtheorem{theorem}{Theorem}
\newtheorem{lemma}{Lemma}[]
\usepackage{nicematrix}
\usepackage{booktabs}
\usepackage{multirow}
\usepackage{subcaption}    
\usepackage{caption}
\definecolor{myred}{RGB}{0,0,0}
\definecolor{myblue}{RGB}{0,0,0}
\definecolor{mygreen}{RGB}{0,0,0}
\definecolor{myblue2}{RGB}{0,0,0}

\def\BibTeX{{\rm B\kern-.05em{\sc i\kern-.025em b}\kern-.08em
    T\kern-.1667em\lower.7ex\hbox{E}\kern-.125emX}}
\begin{document}

\title{FedHiP: Heterogeneity-Invariant Personalized Federated Learning Through Closed-Form Solutions

%
%








}

\vspace{-0.5cm}
\author{\IEEEauthorblockN{Jianheng Tang$^{\scriptstyle{1}}$, Zhirui Yang$^{\scriptstyle{2}}$, Jingchao Wang$^{\scriptstyle{1}}$, Kejia Fan$^{\scriptstyle{2}}$,
Jinfeng Xu$^{\scriptstyle{3}}$,
}
\IEEEauthorblockN{Huiping Zhuang$^{\scriptstyle{4}}$, Anfeng Liu$^{\scriptstyle{2}}$, Houbing Herbert Song$^{\scriptstyle{5}}$, Leye Wang$^{\scriptstyle{1}}$, Yunhuai Liu$^{\scriptstyle{1}}$}
\vspace{0.1cm}
 \IEEEauthorblockA{1. PKU, China \quad  2. CSU, China  \quad  3. HKU, China \quad 4. SCUT, China \quad 5. UMBC, USA
 }
}

\maketitle
\pagestyle{empty}

\begin{abstract}
Lately, Personalized Federated Learning (PFL) has emerged as a prevalent paradigm to deliver personalized models by collaboratively training while simultaneously adapting to each client's local applications.
Existing PFL methods typically face a significant challenge due to the ubiquitous data heterogeneity (i.e., non-IID data) across clients, which severely hinders convergence and degrades performance.
We identify that the root issue lies in the long-standing reliance on gradient-based updates, which are inherently sensitive to non-IID data.
To fundamentally address this issue and bridge the research gap, in this paper, we propose a \underline{H}eterogeneity-\underline{i}nvariant \underline{P}ersonalized \underline{Fed}erated learning scheme, named FedHiP, through analytical (i.e., closed-form) solutions to avoid gradient-based updates.
Specifically, we exploit the trend of self-supervised pre-training, leveraging a foundation model as a frozen backbone for gradient-free feature extraction.
Following the feature extractor, we further develop an analytic classifier for gradient-free training.
To support both collective generalization and individual personalization, our FedHiP scheme incorporates three phases: analytic local training, analytic global aggregation, and analytic local personalization.
The closed-form solutions of our FedHiP scheme enable its ideal property of heterogeneity invariance, meaning that each personalized model remains identical regardless of how non-IID the data are distributed across all other clients.
Extensive experiments on benchmark datasets validate the superiority of our FedHiP scheme, outperforming the state-of-the-art baselines by at least {\color{black}{5.79\%-20.97\%}} in accuracy.
\end{abstract}


\vspace{0.55cm}

\section{Introduction}
Federated Learning (FL) is a popular distributed machine learning paradigm that enables the central server to orchestrate numerous clients’ collaborative training of a global consensus model, while preserving each client's data privacy~\cite{FL-1,FedAvg,FedProx}.
Yet, only one single globally shared model is often insufficient for adapting to each client's local applications, as the clients may have diverse personalized requirements~\cite{PFL-0,Ditto,FedALA}. 
For instance, in an image recognition scenario where one client primarily focuses on identifying animals while another one specializes in recognizing plants, a single global model would likely yield suboptimal results for each individual client.

Therefore, Personalized Federated Learning (PFL) has recently gained prominence as an extension to the FL paradigm, aiming to build personalized models for each individual client alongside the learning of a global model among clients~\cite{PFL-3,PFL-4,PFL-5}.
On the one hand, in PFL, each local client first benefits from the collective generalization by leveraging knowledge from all participating clients within the federation~\cite{PFL-4, PFL-5, PFL-3}. 
On the other hand, PFL also enables individual personalization by adapting models to local data characteristics~\cite{PFL-4, PFL-5, PFL-3}.


\newpage

\begin{figure}[t]
\centering
\includegraphics[width=0.9\columnwidth]{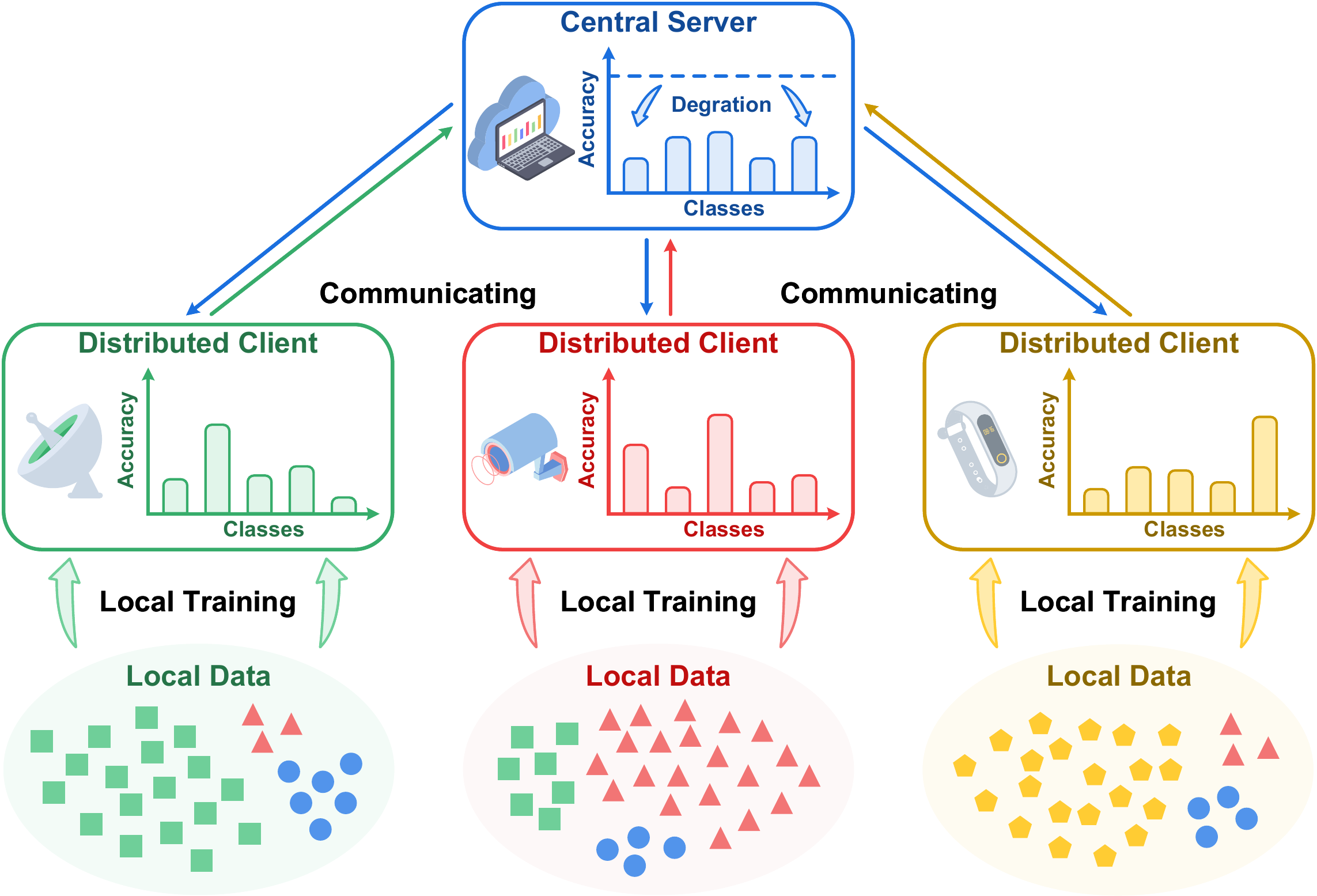}
\caption{Ubiquitous data heterogeneity in PFL. 
Each local model specializes in its unique distribution but lacks generalization ability.
Meanwhile, the global model struggles to learn global knowledge due to the clients' conflicting learning directions.
}
\label{Figure-1}
\vspace{-1.45em}
\end{figure}



A major hurdle of PFL is the ubiquitous challenge of data heterogeneity across clients, often referred to as non-IID data. 
This challenge severely impedes the convergence of federated models and consequently degrades their overall performance.
As illustrated in Fig.~\ref{Figure-1}, because of the clients' personalized preferences, their statistical distributions of data vary with distinct local optima and conflicting learning directions. 
Thus, each local model specializes in its unique local data distribution, causing the global model to deviate from a truly generalized representation.
As a result, the aggregated global model tends to drift away from the truly generalized representation.
This drift directly degrades collective generalization and, even more critically, compromises any subsequent personalization efforts that are based on the compromised global information.



To more thoroughly illustrate the impact of non-IID data on the global aggregation processes of PFL, here, we provide a vivid schematic diagram in Fig.~\ref{fig:Non-iid}.
Specifically, the local gradients of client-side training inherently tend to skew towards the local data distributions~\cite{PFL-0}.
As demonstrated in Fig.~\ref{fig:Non-iid}(a), in an ideal IID environment, each local large gradient general aligns, meaning their overall direction is relatively similar. 
Yet, in the non-IID environment, the local gradients are usually in significantly conflicting directions, as illustrated in Fig.~\ref{fig:Non-iid}(b).
The conflicting local gradients severely impact conventional global aggregation, which is crucial for collaborative training and robust generalization.
Consequently, the aggregated model is often far away from the truly generalized optimum.

\clearpage


\begin{figure}
    \centering
    \includegraphics[width=0.9\linewidth]{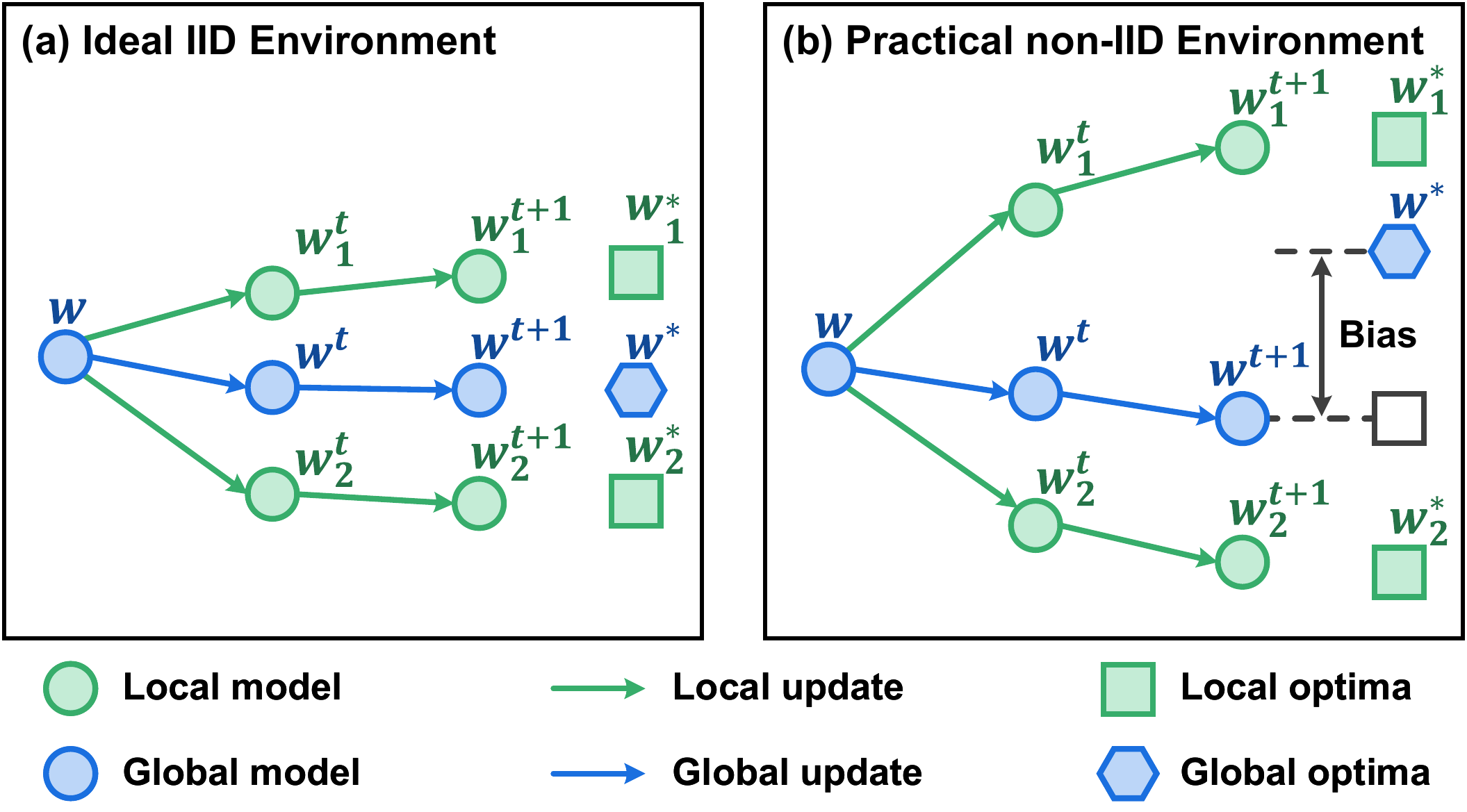}
    \caption{Gradient-based
updates' sensitivity to non-IID data.}
    \label{fig:Non-iid}
    \vspace{-0.45cm}
\end{figure}

This observation reinforces our core insight: \textit{\textbf{the root issue in PFL lies in the long-standing reliance on gradient-based updates, which are just inherently sensitive to non-IID data. }}
Thus, to fundamentally address this gradient-related issue of data heterogeneity, 
we believe that it is necessary to eliminate the primary culprit: gradient-based updates themselves.
Many existing studies for PFL have also recognized the inherent sensitivity of gradient-based updates to non-IID data.
However, they have focused on only mitigating the superficial symptoms rather than directly confronting the root cause~\cite{PFL-0,FedALA,PFL-3}.

To fundamentally address the aforementioned issue and bridge the research gap in PFL, in this paper, we propose a novel \underline{H}eterogeneity-\underline{i}nvariant \underline{P}ersonalized \underline{Fed}erated learning scheme, named FedHiP, though analytical (i.e., closed-form) solutions to avoid gradient-based updates.
Specifically, we exploit the widespread trend of self-supervised pre-training, leveraging a foundation model as a frozen backbone for gradient-free feature extraction.
Following the backbone, we further develop an analytic classifier for gradient-free training.
To balance between collective generalization and individual personalization within our FedHiP scheme, we devise a three-phase analytic framework in a fully gradient-free manner.
The key contributions are summarized as follows.
\begin{enumerate}
    \item Identifying the inherent sensitivity of the gradient-based updates to non-IID data, we propose our FedHiP scheme to fundamentally address this issue by avoiding gradient-based updates via analytical solutions in PFL. 
    \item To support both collective generalization and individual personalization, we devise a three-phase analytic framework for our FedHiP scheme in a gradient-free manner.
    \item We theoretically prove our FedHiP scheme's ideal property of heterogeneity invariance, i.e., each personalized model remains identical regardless of how non-IID the data are distributed across all other clients.
    \item Extensive experiments on various benchmark datasets validate the superiority of our proposed FedHiP scheme, outperforming the baselines by at least 5.79\%-20.97\%.
\end{enumerate}

The rest of this paper is organized as follows.
Section~\ref{Related Work} offers a review of the related work. 
Next, Section~\ref{system model} introduces the system model and problem statement. 
Then, Section~\ref{method} proposes the design of our FedHiP scheme. 
Section~\ref{evaluation}  presents experimental evaluations of our FedHiP scheme extensively.
Finally, Section~\ref{conclusion} gives the conclusion and discussion.

\newpage

\section{Related Work}
\label{Related Work}

\subsection{Personalized Federated Learning}
FL has emerged as a popular distributed machine learning paradigm, enabling multiple clients to collaboratively train a shared global model without exposing their raw data~\cite{FL-1,FL-3}.
In FL, the collective intelligence of the participating clients is leveraged to achieve a common goal, embodying the principles of crowdsensing/crowdsourcing~\cite{a1,a2,a3}.
Unfortunately, this ``one-size-fits-all" approach of traditional FL often falls short when the clients have different underlying data distributions, objectives, or unique local application requirements~\cite{PFL-0,Ditto,FedALA}.

In this context, PFL has recently gained significant attention as an extension of the traditional FL paradigm, aiming to collaboratively train models that are not only globally informed but also individually tailored to each client's specific needs~\cite{PFL-3}. 
Its core appeal lies in its dual capability: to simultaneously foster collective generalization by leveraging knowledge from all participating clients, and to enable individual personalization by adapting models to local data characteristics~\cite{PFL-4,PFL-5,PFL-3}.

Despite the advancements offered by PFL, ubiquitous data heterogeneity across clients still remains a critical challenge that severely hinders convergence and degrades performance.
This challenge is recognized to be attributable to the sensitivity of gradient-based updates to non-IID data, which impedes generalized knowledge in collective aggregation~\cite{PFL-0}.
Although many efforts are made to mitigate these effects, they focus on only mitigating the symptoms rather than confronting the root cause within the gradients~\cite{PFL-0,Ditto,FedALA}.
Differing from prior research, our proposed FedHiP scheme aims to fundamentally avoid gradient-based updates, enabling it to achieve the ideal and rare property of heterogeneity invariance.




\subsection{Analytic Learning}

Analytic learning stands out as a gradient-free technique to tackle common gradient-related challenges, such as vanishing and exploding gradients~\cite{AL_1, ACIL_2, ACIL_1}.
Its characteristic reliance on matrix inversion has also led to its recognition as pseudoinverse learning~\cite{AL_new_0, AL_new_1,AL_2}.
The radial basis network is a classic example of shallow analytic learning, employing the least squares estimation to train parameters subsequent to kernel transformation in its initial layer~\cite{AL_2}.
Beyond shallow architectures, analytic learning has also seen extensive application in multilayer networks by using the least squares methods to linearize portions of nonlinear network training~\cite{AL_3, AL_4, AL_CNN-2, AL_5}.

To ease the memory constraints of analytic learning, the block-wise recursive Moore-Penrose inverse is proposed for efficient joint learning in a manageable recursive manner~\cite{AL_6}.
With this breakthrough, the analytic learning technique has demonstrated exceptional efficacy across diverse applications, e.g., continual learning~\cite{CALM,AFL,AL_RL}.
However, there is still a significant gap in introducing analytic learning into PFL, with the challenge of balancing between collective generalization and individual personalization.
Our FedHiP scheme aims to bridge the research gap, thereby fully leveraging the advantages of gradient-free analytic learning in PFL.
To our best knowledge, we are the first to introduce analytic learning into PFL. 

\clearpage

\section{System Model and Problem Statement}
\label{system model}
In this paper, we consider a typical PFL system operating within a heterogeneous environment, comprising a central server and $K$ distributed clients $\{1, 2, \cdots, K\}$. 
Without loss of generality, we use the most common task of image recognition as a representative example to model the problem, noting that other PFL tasks can be similarly formulated.
For each client $k \in \{1, 2, \cdots, K\}$, we employ $\mathbf{\mathcal{D}}_k \sim \{\mathbf{X}_k, \mathbf{Y}_k\}$ to denote its local training dataset and use $\mathbf{\mathcal{D}}_k^\mathrm{test}$ to denote its local testing dataset. 
The training dataset $\mathbf{\mathcal{D}}_k$ of client $k$ consists of $N_k$ image samples $\mathbf{X}_k\in \mathbb{R}^{N_k \times l\times w\times h}$ and the corresponding labels $\mathbf{Y}_k \in \mathbb{R}^{N_k \times d}$.
Here, $l\times w\times h$ denotes three dimensions of each input image.
Moreover, $\mathbf{Y}_k$ is a one-hot label tensor, and $d$ denotes the total number of classes.
Given the inherent heterogeneity in data distribution among clients, the datasets $\mathcal{D}_1, \mathcal{D}_2, \cdots, \mathcal{D}_K$ are non-IID.
Each client also exhibits distinct preferences in real-world applications, as reflected by their Non-IID test datasets $\mathcal{D}_1^\mathrm{test}, \mathcal{D}_2^\mathrm{test}, \cdots, \mathcal{D}_K^\mathrm{test}$.
Our objective is thus to construct a specific personalized model for each client, maximizing its performance by achieving as high accuracy as possible on its corresponding local application $\mathcal{D}_k^\mathrm{test}$.
This goal requires integrating both the local knowledge (i.e., personalization) from each client $k$'s dataset $\mathcal{D}_k$ and the global knowledge (i.e., generalization) from all clients' datasets $\mathcal{D}_{1:K}$.


\section{The Proposed FedHiP Scheme}
\label{method}

\subsection{Motivation and Overview}
As illustrated in Fig.~\ref{fig:Non-iid}, we demonstrate the effect of client drift caused by Non-IID data by comparing it with an IID data setting.
In the ideal IID scenario, all gradient updates during local training are expected to remain consistent and unified, as shown in Fig.~\ref{fig:Non-iid}(a).
However, in the ubiquitous non-IID setting, each local client updates gradients in significantly conflicting directions, causing the aggregation to substantially deviate from the optimum, as depicted in Fig.~\ref{fig:Non-iid}(b).
Thus, we identify that the root issue is the sensitivity of gradient-based updates to non-IID data.
Motivated by this insight, our FedHiP scheme is designed to fundamentally address the non-IID issue by avoiding reliance on gradient-based updates.

To fulfill the gradient-free goal, we exploit the widespread trend and universal recognition of self-supervised pre-training by employing a foundation model as a frozen backbone.
This frozen backbone facilitates gradient-free feature extraction and benefits from the foundation model's powerful representation capacity.
For instance, in the context of image recognition, we can utilize a Vision Transformer with Masked Auto-Encoders (ViT-MAE) as the foundation model~\cite{ViT-MAE}.
The central server can readily build a foundation model using publicly available datasets, or by directly downloading an open-source model.
Importantly, the self-supervised pre-training process typically involves only image reconstruction and does not require any labeled data for supervision.
Thus, leveraging such a foundation model in PFL doesn't introduce overly strong assumptions, nor does it pose the same privacy risks as typical supervised pre-training models.
Especially, many works have adopted similar approaches and validated their efficacy for FL~\cite{fl-pre-train-1,fl-pre-train-2,fl-pre-train-3,FL-foundation-1,fl-foundation-2}.


Following the frozen feature extractor, we further develop an analytic classifier for gradient-free training. 
Let's denote the extracted feature from the input $\mathbf{X}_k$ as $\mathbf{F}_k$. Our ultimate goal is to build a personalized analytic model $\hat{\mathbf{P}}_k$ for each client $k$ in a gradient-free manner.
To support both collective generalization and individual personalization for constructing $\hat{\mathbf{P}}_k$, we define the following objective for integrated optimization:
\begin{equation}
\label{eq:person-1}
  \underset{\mathbf{{P}}_k}{\min} \;
  {\color{myred} \underbrace{\|\mathbf{Y}_{1:K}-\mathbf{F}_{1:K}\mathbf{P}_k\|^2}_{(a)}}
  + {\color{myblue} \underbrace{\alpha  \|\mathbf{Y}_k-\mathbf{F}_k\mathbf{P}_k\|^2}_{(b)}}
  + {\color{mygreen}\underbrace{\mathbf{\beta} \| \mathbf{P}_k \|^{2}}_{(c)}},
\end{equation}
where $\mathbf{F}_{1:K}$ is the stacked features, and $\mathbf{Y}_{1:K}$ is the stacked labels, assembled from the full datasets $\mathcal{D}_{1:K}$ of all $K$ clients.
Specifically, the first term {\color{myred}$\|\mathbf{Y}_{1:K}-\mathbf{F}_{1:K}\mathbf{P}_k\|^2$} aims to extract the global knowledge (i.e., generalization) from all clients’ full datasets $\mathcal{D}_{1:K}$. 
Meanwhile, the second term {\color{myblue}$\alpha  \|\mathbf{Y}_k-\mathbf{F}_k\mathbf{P}_k\|^2$} is designed to enhance the local knowledge (i.e., personalization) from the target client $k$'s dataset $\mathcal{D}_{1:k}$. 
The hyperparameter {\color{myblue}$\alpha$} serves to control the trade-off between generalization in {\color{myred}$(a)$} and personalization in {\color{myblue}$(b)$}. 
In addition, the third term {\color{mygreen}$\mathbf{\beta} \| \mathbf{P}_k \|^{2}$} is included for regularization to prevent overfitting, with an adjustable hyperparameters {\color{mygreen}$\beta$} for flexibility.


To achieve the objective in \eqref{eq:person-1}, we devise a three-phase analytic framework for our FedHiP scheme, including local training in Section \ref{phase 1}, global aggregation in Section \ref{phase 2}, and local personalization in Section \ref{phase 3}.
The detailed framework of our FedHiP scheme is illustrated in Fig.~\ref{fig_2}.

First of all, each client $k$ is required to train a local model $\hat{\mathbf{L}}_k$ using its local dataset.
This analytic local training phase follows the regularized empirical risk minimization objective in \eqref{eq:local-1}.
This way, each client $k$ can integrate its local knowledge into model $\hat{\mathbf{L}}_k$, thereby providing a foundation for subsequent global knowledge exchange during the second phase.
\begin{equation}
\label{eq:local-1}
  \hat{\mathbf{L}}_k=\arg\underset{\mathbf{{L}}_k}{\min} \ 
  {\color{black}\|\mathbf{Y}_k-\mathbf{F}_k\mathbf{L}_k\|^2}
  +
  {\color{black}\mathbf{\beta} \| \mathbf{L}_k \|^{2}}.
\vspace{-0.5mm}
\end{equation}

Next, each client uploads its local model $\hat{\mathbf{L}}_k$ to the server, which then recursively aggregates the global knowledge.
After aggregating all $K$ clients' knowledge, the server can obtain the final global model $\hat{\mathbf{G}}_{K}$ with the optimization objective in~\eqref{eq:global-1}.
This global aggregation phase primarily aims to minimize the regularized empirical risk across the entire training dataset, under privacy-preserving conditions where the central server is prevented from directly accessing each local dataset.
\begin{equation}
\label{eq:global-1}
  \hat{\mathbf{G}}_{K} =
  \arg\underset{\mathbf{{G}}_K}{\min} \ 
  {\color{black}\|\mathbf{Y}_{1:K}-\mathbf{F}_{1:K}\mathbf{G}_K\|^2}
  +
  {\color{black}\mathbf{\beta} \| \mathbf{G}_K \|^{2}}.
\vspace{-0.5mm}
\end{equation}

Then, the central server distributes the global information back to each client. 
Finally, each client $k$ performs additional local personalization to get its personalized model $\hat{\mathbf{P}}_k$. 
Thus, each obtained model $\hat{\mathbf{P}}_k$ fulfills the goal in~\eqref{eq:person-1} 
by balancing collective generalization and individual personalization.

The entire three phases are fully gradient-free via analytical (i.e., closed-form) solutions, enabling our FedHiP scheme's property of heterogeneity invariance.
This property means that each personalized model remains identical regardless of how non-IID the data are distributed across all other clients.
We provide detailed theoretical analyses in Section~\ref{theoretical}.


\begin{figure*}
\centering
\includegraphics[width=2.0\columnwidth]{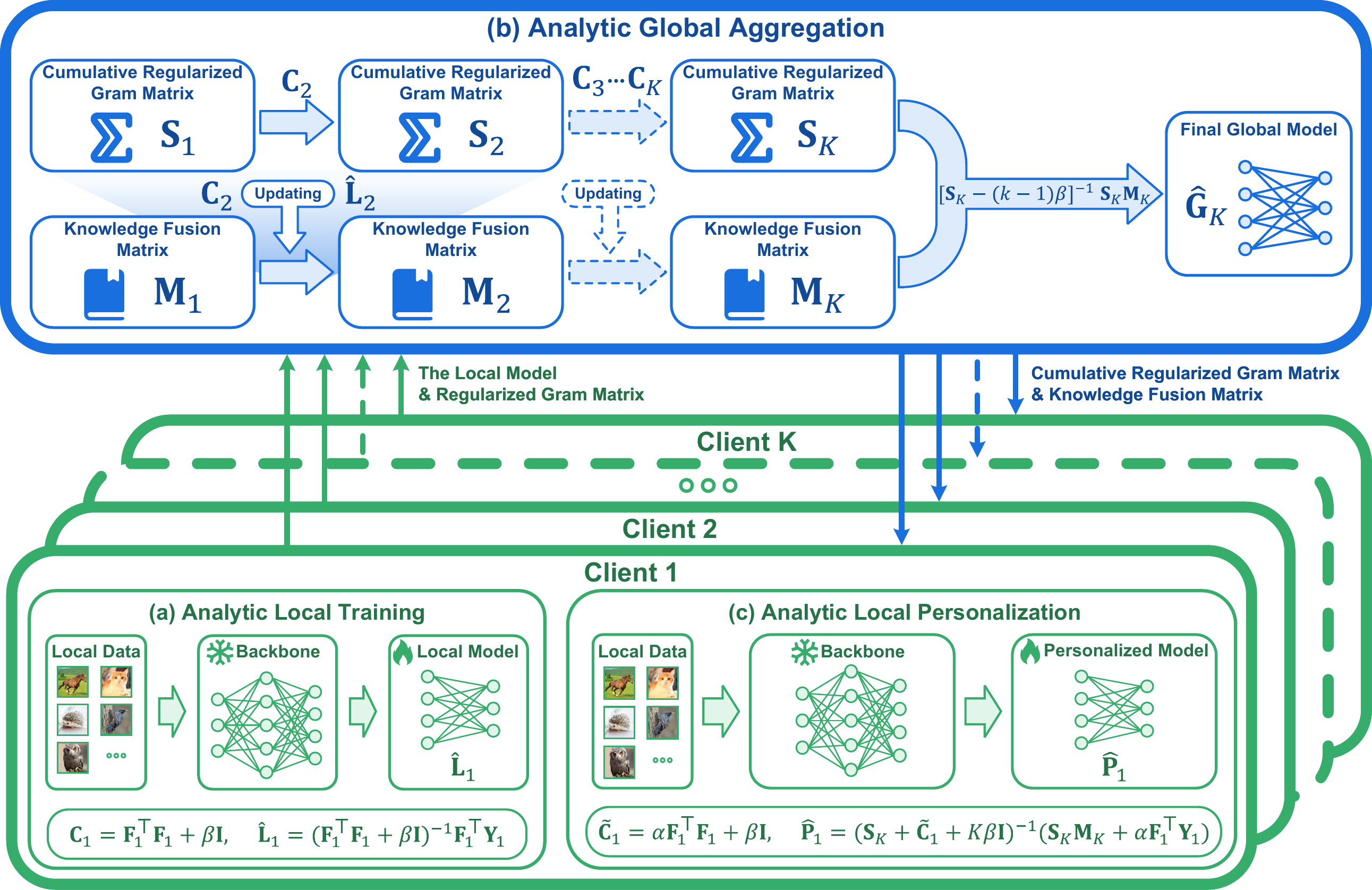}
\vspace{0.15cm}
\caption{The detailed design of our proposed FedHiP scheme.}
\label{fig_2}
\vspace{0.01cm}
\end{figure*}


%


\subsection{Phase 1: Analytic Local Training}
\label{phase 1}






As a preliminary step, each participating client first registers with the central server and downloads a foundation model as the feature extractor.
Taking image recognition as an example, we can employ ViT-MAE~\cite{ViT-MAE} as the foundation model, where the clients only need to download the encoder component for local feature extraction.
Specifically, for the $k$-th client, upon receiving the foundation model backbone, $\mathrm{Backbone}(\cdot, \boldsymbol{\Theta})$, it performs forward propagation on its local raw data $\mathbf{X}_{k}$ through the backbone network to extract the feature matrix $\mathbf{F}_{k}$:
\begin{equation}
\label{eq:backbone-1}
\mathbf{F}_{k} = \mathrm{Backbone}\left( {\mathbf{X}_{k};\boldsymbol{\Theta}} \right),
\end{equation}
where $\mathbf{F}_{k} \in \mathbb{R}^{N_k \times m}$ represents the feature matrix of the $k$-th client, with $N_k$ and $m$ denoting the sample size and embedding dimension for the $k$-th client, respectively.

Subsequently, each distributed client $k$ conducts the first phase by analytic local training based on the extracted feature representations $\mathbf{F}_{k}$, as illustrated in Fig.~\ref{fig_2}. 
For each client $k$, the first phase's objective is to learn a local model $\hat{\mathbf{L}}_k \in \mathbb{R}^{m \times d}$ that satisfies the optimization objective formulated in~\eqref{eq:local-1}.
Here, in our FedHiP scheme, we leverage the least squares method to obtain the closed-form solution of the local model $\hat{\mathbf{L}}_k$ for each client $k$, in a gradient-free manner.
Specifically the closed-form expression for the optimal solution of~\eqref{eq:local-1} is given by~\eqref{eq:local_solution}.


\begin{equation}
\hat{\mathbf{L}}_k = \left( \mathbf{F}_k^\top \mathbf{F}_k + \beta \mathbf{I} \right)^{-1} \mathbf{F}_k^\top \mathbf{Y}_k.
\label{eq:local_solution}
\vspace{-1.1mm}
\end{equation}

Here, we define the \textit{Regularized Gram Matrix} $\mathbf{C}_k \in \mathbb{R}^{m \times m}$ for the $k$-th client, and it can be computed as follows:
\begin{equation}
\mathbf{C}_k = \mathbf{F}_k^{\top}\mathbf{F}_k + \beta \mathbf{I}.
\label{eq:covariance_matrix}
\vspace{-1.1mm}
\end{equation}

Consequently, the solution of the local model $\hat{\mathbf{L}}_k$ in \eqref{eq:local_solution} can be further expressed in the compact form with $\mathbf{C}_k$, as follows:
\begin{equation}
\hat{\mathbf{L}}_k = \mathbf{C}_k^{-1} \mathbf{F}_k^\top \mathbf{Y}_k.
\label{eq:local_solution-2}
\vspace{-1.1mm}
\end{equation}

After completing the aforementioned analytic computations, the $k$-th client transmits the obtained local model $\hat{\mathbf{L}}_{k}$ and the \textit{Regularized Gram Matrix} $\mathbf{C}_k$ to the central server.
Notably, these transmitted matrices convey compressed local knowledge without compromising each client's raw data privacy, as they cannot be reverse-engineered to reconstruct $\mathbf{F}_k$ and $\mathbf{Y}_k$.
Given $\hat{\mathbf{L}}_{k}$ and $\mathbf{C}_k$, there exist infinitely many \textit{indistinguishable} pairs of $\mathbf{F}'_k=\mathbf{U}\mathbf{F}_k$ and $\mathbf{Y}'_k=\mathbf{U}\mathbf{Y}_k$, where $\mathbf{U}$ can represent any \textit{semi-orthogonal matrix} such that $\mathbf{U}^\top \mathbf{U}=\mathbf{I}$, as shown in~\eqref{eq:local_solution-privacy}.
\begin{equation}
\left\{
\begin{aligned}
&\mathbf{C}'_k = (\mathbf{U}
\mathbf{F}_k)^{\top}(\mathbf{U}\mathbf{F}_k) + \beta \mathbf{I}
=\mathbf{F}_k^{\top}\mathbf{F}_k + \beta \mathbf{I}
=\mathbf{C}_k.
\\
&\hat{\mathbf{L}'}_k = (\mathbf{C'}_k)^{-1} (\mathbf{U}\mathbf{F}_k)^\top (\mathbf{U}\mathbf{Y}_k)
=\mathbf{C}_k^{-1} \mathbf{F}_k^\top \mathbf{Y}_k
=\hat{\mathbf{L}}_k.
\end{aligned}
\right.
\label{eq:local_solution-privacy}
\vspace{-1.0mm}
\end{equation}

Therefore, the feature matrix $\mathbf{F}_k$ and label matrix $\mathbf{Y}_k$ cannot be recovered because of the infinite possible \textit{semi-orthogonal transformations}, ensuring the privacy of each client's raw data.

\clearpage

\subsection{Phase 2: Analytic Global Aggregation}
\label{phase 2}

%






Upon completion of analytic local training by all clients, the server then performs analytic global aggregation using the local knowledge received from each client $k$ (i.e., $\mathbf{C}_k$ and $\hat{\mathbf{L}}_k$).
In this phase, for each client $k$, the server recursively computes and retains the \textit{Cumulative Regularized Gram Matrix} $\mathbf{S}_k$ and the \textit{Knowledge Fusion Matrix} $\mathbf{M}_k$, thereby ultimately deriving the global model $\hat{\mathbf{G}}_k$ that satisfies the objective in~\eqref{eq:global-1}.

First of all, the server recursively computes the \textit{Cumulative Regularized Gram Matrix} $\mathbf{S}_k \in \mathbb{R}^{m \times m}$, as follows:
\begin{equation} 
\mathbf{S}_0=0, \quad \mathbf{S}_k = \mathbf{S}_{k-1} + \mathbf{C}_k = \sum_{i=1}^{k} \mathbf{C}_i.
\label{eq:recursive-C}
\vspace{-0.5mm}
\end{equation}

Then, the server  calculates the \textit{Knowledge Fusion Matrix} $\mathbf{M}_k \in \mathbb{R}^{m \times d}$ to update the global knowledge for incorporating the $k$-th client's local knowledge $\hat{\mathbf{L}}_k$, as follows:
\begin{equation} 
\mathbf{M}_k = \boldsymbol{\mu}_k \mathbf{M}_{k-1} + \boldsymbol{\nu}_k \hat{\mathbf{L}}_k,
\label{eq:recursive-weights}
\vspace{-0.5mm}
\end{equation}
where the weighting matrices $\boldsymbol{\mu}_k$ and $\boldsymbol{\nu}_k$ are defined in~\eqref{eq:coefficient}.
\begin{equation}
\begin{aligned}
\boldsymbol{\mu}_k = \mathbf{S}^{-1}_{k} \mathbf{S}_{k-1},
\quad
\boldsymbol{\nu}_k = \mathbf{S}^{-1}_{k} \mathbf{C}_{k}.
\end{aligned}
\label{eq:coefficient}
\vspace{-0.5mm}
\end{equation}

The recursive form~\eqref{eq:recursive-weights} of our FedHiP scheme characterizes the knowledge update dynamics in an explicit and interpretable manner. 
Specifically, the first term $\boldsymbol{\mu}_k \mathbf{M}_{k-1}$ corresponds to the weighted retention of previously accumulated knowledge, while the second term $\boldsymbol{\nu}_k \hat{\mathbf{L}}_k$ represents the weighted incorporation of new knowledge from the $k$-th client.
Moreover, the weighting matrices $\boldsymbol{\mu}_k$ and $\boldsymbol{\nu}_k$ quantify the importance of historically retained and newly acquired knowledge, respectively.

In our FedHiP, the \textit{Knowledge Fusion Matrix} $\mathbf{M}_k$ serves as a dedicated component on the central server for aggregating and memorizing local knowledge from clients.
In particular, the \textit{Knowledge Fusion Matrix} $\mathbf{M}_k$ is not the final global model itself. 
Instead, the server needs to invoke this knowledge to derive the global model $\hat{\mathbf{G}}_k \in \mathbb{R}^{m \times m}$ as follows:
%
\begin{equation} 
\hat{\mathbf{G}}_k =  \left[\mathbf{S}_k - (k - 1) \beta \mathbf{I}\right]^{-1} \mathbf{S}_k \mathbf{M}_k.
\label{eq:correct}
\vspace{-0.5mm}
\end{equation}

As illustrated in Fig.~\ref{fig_2}, the server iteratively executes this computation until it integrates the local model knowledge from all $K$ clients and obtains the final global model $\hat{\mathbf{G}}_K$ along with the corresponding \textit{Knowledge Fusion Matrix} $\mathbf{M}_K$.
Moreover, in \textbf{Theorem 1} from Section \ref{theoretical}, we show that the final global model $\hat{\mathbf{G}}_K$ within our FedHiP scheme is exactly equivalent to performing regularized empirical risk minimization~\eqref{eq:global-1} on the complete dataset $\mathcal{D}_{1:K}$ aggregated from all $K$ clients.

\subsection{Phase 3: Analytic Local Personalization}
\label{phase 3}








After finishing the second phase of our FedHiP scheme, the final global model effectively assimilates knowledge from diverse client training data, thereby exhibiting superior generalization capabilities of global knowledge. 
However, as we have analyzed earlier, only a single shared model is insufficient for PFL. 
Therefore, for each client $k$, we further design the third phase to construct the personalized model $\hat{\mathbf{P}}_k$ to fulfill the integrated optimization objective in~\eqref{eq:person-1}, enabling individual personalization by adapting to the local distribution.



Specifically, in this phase, the central server first distributes the \textit{Cumulative Regularized Gram Matrix} $\mathbf{S}_K$ and \textit{Knowledge Fusion Matrix} $\mathbf{M}_K$ to all clients.
Then, each client $k$ computes the local \textit{Weighted Regularized Gram Matrix} $\widetilde{\mathbf{C}}_k$ as follows:
\begin{equation}
\widetilde{\mathbf{C}}_k = \alpha\mathbf{F}_k^{\top}\mathbf{F}_k + \beta \mathbf{I}.
\label{eq:covariance_matrix-p}
\end{equation}

Thus, based on $\mathbf{S}_K$ and $\mathbf{M}_K$ received from the server, along with the local matrices $\mathbf{F}_k$, $\mathbf{Y}_k$, and $\widetilde{\mathbf{C}}_k$, each client $k$ can get its personalized model $\hat{\mathbf{P}}_k$ through an analytic solution of~\eqref{eq:final_correction-p}.
\begin{equation}
\hat{\mathbf{P}}_k = 
({\mathbf{S}}_K +\widetilde{\mathbf{C}}_k - K \beta \mathbf{I})^{-1}
(\mathbf{S}_K \mathbf{M}_K + \alpha \mathbf{F}_k^{\top}\mathbf{Y}_k).
\label{eq:final_correction-p}
\end{equation}

The obtained $\hat{\mathbf{P}}_k$ is equivalent to the optimal solution of the empirical risk minimization objective within~\eqref{eq:person-1}, as established in \textbf{Theorem 2} from Section \ref{theoretical}.
The closed-form solutions of our FedHiP scheme enable its ideal property of heterogeneity invariance, as shown in \textbf{Theorem 3} from Section \ref{theoretical}.
This ideal property means that each personalized model (i.e., $\hat{\mathbf{P}}_k$ for the $k$-th client) remains identical regardless of how non-IID the data are distributed across all other clients.\
For clarity, we provide the pseudocode of our FedHiP scheme in Algorithm~\ref{alg:fedhip}.


\begin{algorithm}[t]
\caption{Our proposed FedHiP scheme}
\label{alg:fedhip}
\begin{algorithmic}[1]
\STATE \textbf{Input:} The foundation model $\mathrm{Backbone}(\cdot, \boldsymbol{\Theta})$, the local training datasets $\{\mathcal{D}_k = (\mathbf{X}_k, \mathbf{Y}_k)\}_{k=1}^K$.
\STATE \textbf{Output:} 
The personalized models $\{\hat{\mathbf{P}}_k\}_{k=1}^K$.


\STATE \textbf{// Phase 1: Analytic Local Training (Client)}
\FOR{each client $k = 1, 2, \cdots, K$ \textbf{in parallel}}
    \STATE Extract local features $\mathbf{F}_k$ through backbone via \eqref{eq:backbone-1};
    \STATE Compute the \textit{Regularized Gram Matrix} $\mathbf{C}_k$ via (\ref{eq:covariance_matrix});
    \STATE Obtain the local model $\hat{\mathbf{L}}_k$ via \eqref{eq:local_solution-2};
    \STATE Send $\{\mathbf{C}_k, \hat{\mathbf{L}}_k\}$ to \textit{Server};
\ENDFOR

\STATE \textbf{// Phase 2: Analytic Global Aggregation (Server)}
\STATE Initialize $\mathbf{S}_0 = \mathbf{0}$, $\mathbf{M}_0 = \mathbf{0}$;
\FOR{each client $k = 1, 2, \cdots, K$}
    \STATE Receive the local knowledge $\{\mathbf{C}_k, \hat{\mathbf{L}}_k\}$ from \textit{Client $k$};
    \STATE Get \textit{Cumulative Regularized Gram Matrix} $\mathbf{S}_k$ via (\ref{eq:recursive-C});
    \STATE Calculate weighting matrices $\boldsymbol{\mu}_k$ and $\boldsymbol{\nu}_k$ via (\ref{eq:coefficient});
    \STATE Update \textit{Knowledge Fusion Matrix} $\mathbf{M}_k$ via (\ref{eq:recursive-weights});
    \STATE Derive global model $\hat{\mathbf{G}}_k$ based on $\mathbf{S}_k$ and $\mathbf{M}_k$ via (\ref{eq:correct});
\ENDFOR

\STATE \textbf{// Phase 3: Analytic Local Personalization (Client)}
\FOR{each client $k = 1, 2, \cdots, K$ \textbf{in parallel}}
    \STATE Receive the global matrices $\{\mathbf{S}_K, \mathbf{M}_K\}$ from \textit{Server};
    \STATE Get \textit{Weighted Regularized Gram Matrix} $\widetilde{\mathbf{C}}_k$ via (\ref{eq:covariance_matrix-p});
    \STATE Derive the personalized model $\hat{\mathbf{P}}_k$ via (\ref{eq:final_correction-p});
\ENDFOR

\STATE \textbf{Return:} The personalized models $\{\hat{\mathbf{P}}_k\}_{k=1}^K$.
\end{algorithmic}
\end{algorithm}



\subsection{Theoretical Analyses}
\label{theoretical}

\begin{lemma}
\label{lemma1}
    The \textit{Knowledge Fusion Matrix} $\mathbf{S}_K$, derived by the recursive formulation~\eqref{eq:recursive-weights} in our FedHiP scheme, is mathematically equivalent to the following closed-form solution:
\begin{equation}
\begin{aligned} 
\mathbf{M}_K = \left( \mathbf{F}_{1:K}^\top \mathbf{F}_{1:K} + K\beta \mathbf{I} \right)^{-1} \mathbf{F}_{1:K}^\top \mathbf{Y}_{1:K}.
\end{aligned} 
\end{equation}
\end{lemma}

\begin{proof}
We proceed by systematically expanding the recursive relation~\eqref{eq:recursive-weights}. Beginning with the recurrence relation, we have:
\begin{align}
\mathbf{M}_K &= \boldsymbol{\mu}_K \mathbf{M}_{K-1} + \boldsymbol{\nu}_K \hat{\mathbf{L}}_K \\
&= \boldsymbol{\mu}_K (\boldsymbol{\mu}_{K-1} \mathbf{M}_{K-2} + \boldsymbol{\nu}_{K-1} \hat{\mathbf{L}}_{K-1}) + \boldsymbol{\nu}_K \hat{\mathbf{L}}_K \\
&= \boldsymbol{\mu}_K \boldsymbol{\mu}_{K-1} \mathbf{M}_{K-2} + \boldsymbol{\mu}_K \boldsymbol{\nu}_{K-1} \hat{\mathbf{L}}_{K-1} + \boldsymbol{\nu}_K \hat{\mathbf{L}}_K.
\end{align}

By iterating this expansion to its conclusion, we can obtain the general form of $\mathbf{M}_K$, as follows:
\begin{equation}
\mathbf{M}_K = \sum_{j=1}^K \left(\prod_{i=j+1}^K \boldsymbol{\mu}_i\right) \boldsymbol{\nu}_j \hat{\mathbf{L}}_j.
\end{equation}

Due to the telescoping property of $\boldsymbol{\mu}_k$ within \eqref{eq:coefficient}, we can derive its product telescopes as follows:
\begin{equation}
\prod_{i=j+1}^K \boldsymbol{\mu}_i = \prod_{i=j+1}^K \mathbf{S}_i^{-1} \mathbf{S}_{i-1} = \mathbf{S}_K^{-1} \mathbf{S}_j.
\end{equation}

Substituting this telescoping result and using the definition of $\boldsymbol{\nu}_j$ in \eqref{eq:coefficient}, we can derive:
\begin{equation}
\mathbf{M}_K = 
\sum_{j=1}^K \mathbf{S}_K^{-1} \mathbf{S}_j \boldsymbol{\nu}_j \hat{\mathbf{L}}_j
=
\mathbf{S}_K^{-1} \sum_{j=1}^K \mathbf{C}_j \hat{\mathbf{L}}_j.
\label{proof-MK-1}
\end{equation}

Using the equation~\eqref{eq:local_solution-2} for $\hat{\mathbf{L}}_j$, we can establish that $\mathbf{C}_j \hat{\mathbf{L}}_j = \mathbf{F}_j^\top \mathbf{Y}_j$. 
Substituting this result into~\eqref{proof-MK-1}, we can thus derive:
\begin{equation}
\mathbf{M}_K = \mathbf{S}_K^{-1} \mathbf{F}_{1:K}^\top \mathbf{Y}_{1:K}.
\end{equation}

Using the definition of $\mathbf{S}_K$ in~\eqref{eq:recursive-C}, we have $\mathbf{S}_k = \mathbf{F}_{1:k}^\top \mathbf{F}_{1:k} + k\beta \mathbf{I}$. Therefore, the closed-form expression of $\mathbf{M}_K$ becomes:
\begin{equation}
\mathbf{M}_K = \left( \mathbf{F}_{1:K}^\top \mathbf{F}_{1:K} + K\beta \mathbf{I} \right)^{-1} \mathbf{F}_{1:K}^\top \mathbf{Y}_{1:K}.
\end{equation}
\end{proof}

\begin{theorem}\label{thm:1}
The final global model $\hat{\mathbf{G}}_K$ derived by~\eqref{eq:correct} in our FedHiP scheme is mathematically equivalent to the optimal solution of the empirical risk minimization objective within~\eqref{eq:global-1}.

\end{theorem}

\begin{proof}
To establish the desired equivalence, we first derive the optimal solution of~\eqref{eq:global-1} through least squares by taking its partial derivative with respect to $\mathbf{G}_k$ and setting it to zero:
\begin{equation}
\hat{\mathbf{G}}_K = \left( \mathbf{F}_{1:K}^{\top}\mathbf{F}_{1:K} + \beta \mathbf{I} \right)^{-1} \mathbf{F}_{1:K}^\top \mathbf{Y}_{1:K}.
\label{least squares-1}
\end{equation}

Based on Lemma~\ref{lemma1}, the \textit{Knowledge Fusion Matrix} satisfies:
\begin{equation}
\mathbf{M}_K = \left( \mathbf{F}_{1:K}^\top \mathbf{F}_{1:K} + K\beta \mathbf{I} \right)^{-1} \mathbf{F}_{1:K}^\top \mathbf{Y}_{1:K}.
\end{equation}


Since $\mathbf{S}_K = \mathbf{F}_{1:K}^{\top}\mathbf{F}_{1:K} + K\beta \mathbf{I}$, we can obtain:
\begin{equation}
\left\{
\begin{aligned}
&\mathbf{S}_K - (K - 1) \beta \mathbf{I} = \mathbf{F}_{1:K}^{\top}\mathbf{F}_{1:K} + \beta \mathbf{I}.
\\
&\mathbf{S}_k \mathbf{M}_k = \mathbf{S}_k \mathbf{S}^{-1}_k \mathbf{F}_{1:K}^{\top}\mathbf{F}_{1:K} = \mathbf{F}_{1:K}^{\top}\mathbf{F}_{1:K}
\end{aligned}
\right.
\label{proof-thm1-1}
\end{equation}

By substituting this conclusion~\eqref{proof-thm1-1} into the formula~\eqref{eq:correct}, the closed-form expression of $\hat{\mathbf{G}}_K$ becomes:
\begin{equation}
\hat{\mathbf{G}}_K = \left( \mathbf{F}_{1:K}^{\top}\mathbf{F}_{1:K} + \beta \mathbf{I} \right)^{-1} \mathbf{F}_{1:K}^\top \mathbf{Y}_{1:K},
\end{equation}
proving the equivalence with the optimal
solution of \eqref{eq:global-1}.
\end{proof}
\begin{theorem}\label{thm:2}
The personalized model $\hat{\mathbf{G}}_k$ through~\eqref{eq:final_correction-p} in our FedHiP scheme is mathematically equivalent to the optimal solution of the empirical risk minimization objective 
within~\eqref{eq:person-1}.
\end{theorem}


\begin{proof}
To establish the desired equivalence, we first derive the optimal solution of~\eqref{eq:person-1} through least squares by taking its partial derivative with respect to $\mathbf{P}_k$ and setting it to zero:
\begin{small}
\begin{equation*}
\hat{\mathbf{P}}_k = \left( \mathbf{F}_{1:K}^\top\mathbf{F}_{1:K} + \alpha\mathbf{F}_k^\top\mathbf{F}_k + \beta\mathbf{I} \right)^{-1} \left( \mathbf{F}_{1:K}^\top\mathbf{Y}_{1:K} + \alpha\mathbf{F}_k^\top\mathbf{Y}_k \right).
\label{eq:personalized_least_squares}
\vspace{-0.5mm}
\end{equation*}
\end{small}

Based on the established formula in~\eqref{eq:recursive-C} and~\eqref{eq:covariance_matrix-p}, we have:
\begin{equation}
\left\{
\begin{aligned}
\mathbf{S}_K &= \mathbf{F}_{1:K}^\top\mathbf{F}_{1:K} + K\beta\mathbf{I},  \\
\widetilde{\mathbf{C}}_k &= \alpha\mathbf{F}_k^\top\mathbf{F}_k + \beta\mathbf{I}. 
\end{aligned}
\right.
\label{eq:Ck_def}
\vspace{-1.0mm}
\end{equation}

Thus, we can further derive:
\begin{equation}
\mathbf{S}_K + \widetilde{\mathbf{C}}_k - K\beta\mathbf{I} = \mathbf{F}_{1:K}^\top\mathbf{F}_{1:K} + \alpha\mathbf{F}_k^\top\mathbf{F}_k + \beta\mathbf{I}.
\label{proof-thm2-1}
\vspace{-0.5mm}
\end{equation}



Combining the established conclusion in \eqref{proof-thm1-1}, we have
\begin{equation}
\mathbf{S}_K \mathbf{M}_K + \alpha \mathbf{F}_k^{\top}\mathbf{Y}_k = \mathbf{F}_{1:K}^\top\mathbf{Y}_{1:K} + \alpha\mathbf{F}_k^\top\mathbf{Y}_k.
\label{proof-thm2-2}
\vspace{-0.5mm}
\end{equation}

Thus, by substituting the conclusions~\eqref{proof-thm2-1} and~\eqref{proof-thm2-2} into the formula~\eqref{eq:final_correction-p}, the closed-form expression of $\hat{\mathbf{P}}_K$ becomes:
\begin{small}
\begin{equation*}
\hat{\mathbf{P}}_k = \left( \mathbf{F}_{1:K}^\top\mathbf{F}_{1:K} + \alpha\mathbf{F}_k^\top\mathbf{F}_k + \beta\mathbf{I} \right)^{-1} \left( \mathbf{F}_{1:K}^\top\mathbf{Y}_{1:K} + \alpha\mathbf{F}_k^\top\mathbf{Y}_k \right),
\end{equation*}
\end{small}
proving the equivalence with the optimal solution of~\eqref{eq:person-1}.
\end{proof}

Based on the obtained closed-form solutions, we then derive our FedHiP scheme's heterogeneity-invariance property below.

\begin{theorem}[\textbf{Heterogeneity-Invariance Property}]

Consider a federated system with $K$ clients, where the data distributions of the clients are denoted as
$\mathcal{D}_{1:K} = \{\mathcal{D}_1, \cdots, \mathcal{D}_k, \cdots, \mathcal{D}_K\}$. Let the personalized model for any client $k \in \{1, 2, \dots, K\}$ in our FedHiP scheme be denoted as $\hat{\mathbf{P}}_k= \hat{\mathbf{P}}(\mathcal{D}_{k}, \{\mathcal{D}_j\}_{j \neq k})$, which can be viewed as a function dependent on the clients' data distributions.
Next, consider an alternative distribution with diverse heterogeneity across the other clients while fixing the client $k$’s data $\mathcal{D}_k$, i.e.,
$\mathcal{D}'_{1:K} = \{\mathcal{D}'_1, \cdots, \mathcal{D}_k, \cdots, \mathcal{D}'_K\}$.
We denote personalized model for the client $k$ under this new configuration as $\hat{\mathbf{P}}'_k = \hat{\mathbf{P}}(\mathcal{D}_k, \{\mathcal{D}'_j\}_{j \neq k})$.
Then, we can always obtain that $\hat{\mathbf{P}}'_k$ is mathematically equivalent to $\hat{\mathbf{P}}_k$, i.e.,
$$
\hat{\mathbf{P}}(\mathcal{D}_k, \{\mathcal{D}_j\}_{j \neq k}) = \hat{\mathbf{P}}(\mathcal{D}_k, \{\mathcal{D}'_j\}_{j \neq k}).
$$

\end{theorem}

\begin{proof}

Here, let $\mathbf{F}'_{1:K}$ and $\mathbf{Y}'_{1:K}$ denote the complete feature matrices extracted from $\mathcal{D}'_{1:K}$ via \eqref{eq:backbone-1} and the corresponding label matrices, respectively.
Since $\mathcal{D}_{1:K}$ and $\mathcal{D}'_{1:K}$ share the same overall data pool across clients but differ in their data distributions, $\mathcal{D}'_{1:K}$ can be viewed as a permutation of $\mathcal{D}_{1:K}$ with a different sample-level reordering.
Thus, we can obtain:
\begin{equation}
    \mathbf{F}'_{1:K} = \mathbf{A}\mathbf{F}_{1:K}, \quad \mathbf{Y}'_{1:K} = \mathbf{A}\mathbf{Y}_{1:K},
\vspace{-0.5mm}
\end{equation}
where $\mathbf{A}$ represents the associated permutation matrix that embodies the row transformation operation, satisfying $\mathbf{A}^\top \mathbf{A} = \mathbf{I}$.

As established in \textbf{Theorem 2}, the obtained personalized model $\hat{\mathbf{P}}(\mathcal{D}_k, \{\mathcal{D}_j\}_{j \neq k})$ is determined by four terms: $\mathbf{F}_{k}^\top\mathbf{F}_{k}$, $\mathbf{F}_{k}^\top\mathbf{Y}_{k}$, $\mathbf{F}_{1:K}^\top\mathbf{F}_{1:K}$, and $\mathbf{F}_{1:K}^\top\mathbf{Y}_{1:K}$.
Since the client $k$'s local data $\mathcal{D}_k$ is fixed, it follows that the first two terms are identical when comparing $\hat{\mathbf{P}}(\mathcal{D}_k, \{\mathcal{D}_j\}_{j \neq k})$ and $\hat{\mathbf{P}}(\mathcal{D}_k, \{\mathcal{D}'_j\}_{j \neq k})$.

By using the orthogonality of the permutation matrix $\mathbf{A}$, we prove the invariance of the \textit{auto-correlation term} $\mathbf{F}_{1:K}^\top\mathbf{F}_{1:K}$:
\begin{equation}
    \mathbf{F}_{1:K}^{'\top} \mathbf{F}'_{1:K} = \mathbf{F}_{1:K}^\top \mathbf{A}^\top \mathbf{A} \mathbf{F}_{1:K} = \mathbf{F}_{1:K}^\top\mathbf{F}_{1:K}.
\end{equation}

Similarly, we can further demonstrate the invariance of the \textit{cross-correlation term} $\mathbf{F}_{1:K}^\top\mathbf{Y}_{1:K}$ as follows:
\begin{equation}
    \mathbf{F}_{1:K}^{'\top} \mathbf{Y}'_{1:K} = \mathbf{F}_{1:K}^\top \mathbf{P}^\top \mathbf{P} \mathbf{Y}_{1:K} = \mathbf{F}_{1:K}^\top\mathbf{Y}_{1:K}.
\end{equation}

In summary, regardless of how the other clients' data distributions change, the four terms $\mathbf{F}_{k}^\top\mathbf{F}_{k}$, $\mathbf{F}_{k}^\top\mathbf{Y}_{k}$, $\mathbf{F}_{1:K}^\top\mathbf{F}_{1:K}$, and $\mathbf{F}_{1:K}^\top\mathbf{Y}_{1:K}$ all remain invariant.
Therefore, the personalized model obtained by our FedHiP scheme is also invariant.
\end{proof}

Moreover, we conduct an efficiency analysis of our FedHiP scheme, examining its running overhead of both computational and communication, as demonstrated below.

\textbf{(1) Computational Overhead:}

Our FedHiP scheme can achieve significant computational efficiency by relying on the lightweight forward-propagation rather than the burdensome back-propagation.
For each of the three phases, we analyze the computational overhead below.

\textbf{Analytic Local Training:} 
In this phase, each client $k$ first computes the \textit{Regularized Gram Matrix} $\mathbf{C}_k \in \mathbb{R}^{m \times m}$ using~\eqref{eq:covariance_matrix} with a computational complexity of $\mathcal{O}(m^2 N_k)$.
Subsequently, the client solves for its local model $\hat{\mathbf{L}}_k \in \mathbb{R}^{m \times d}$ using~\eqref{eq:local_solution-2} with a computational complexity of $\mathcal{O}(m^3 + N_k m d + m^2 d)$.

\textbf{Analytic Global Aggregation:} In this phase, the server recursively updates the cumulative matrix $\mathbf{S}_k \in \mathbb{R}^{m \times m}$ and the \textit{Knowledge Fusion Matrix} $\mathbf{M}_k \in \mathbb{R}^{m \times d}$ based on~\eqref{eq:recursive-C} and~\eqref{eq:recursive-weights}, which incur the complexities of $\mathcal{O}(m^2)$ and $\mathcal{O}(m^3 + m^2 d)$, respectively.
Subsequently, the final global model $\hat{\mathbf{G}}_k \in \mathbb{R}^{m \times d}$ is derived from~\eqref{eq:correct} with a complexity of $\mathcal{O}(m^3 + m^2 d)$.
The recursive process for aggregation is repeated for $K$ iterations.

\textbf{Analytic Local Personalization:} In this phase, each client constructs its weighted matrix $\widetilde{\mathbf{C}}_k \in \mathbb{R}^{m \times m}$ via~\eqref{eq:covariance_matrix-p} with a complexity of $\mathcal{O}(m^2 N_k)$, and subsequently computes its personalized model $\hat{\mathbf{P}}_k \in \mathbb{R}^{m \times d}$ according to~\eqref{eq:final_correction-p}, requiring a computational complexity of $\mathcal{O}(m^3 + N_k m d + m^2 d)$.

In summary, in our FedHiP scheme, each client has a total computational complexity of $\mathcal{O}(m^3 + N_k m d + m^2 d+ m^2 N_k)$, while the server incurs a complexity of $\mathcal{O}(Km^3 + Km^2 d)$.




\textbf{(2) Communication Overhead:}

By avoiding the iterative gradient-based updates via closed-form solutions in PFL, our FedHiP scheme achieves significant communication efficiency with only a single communication round.
Specifically, the communications between clients and the server involve both uploading and downloading:

\textbf{Uploading Communication:} After the local training, each client $k$ uploads its \textit{Regularized Gram Matrix} $\mathbf{C}_k \in \mathbb{R}^{m \times m}$ and local model $\hat{\mathbf{L}}_k \in \mathbb{R}^{m \times d}$ to the central server, with a per-client uploading complexity of $\mathcal{O}(m^2 + md)$.

\textbf{Downloading Communication:} After the global aggregation, the central server distributes the \textit{Cumulative Regularized Gram Matrix} $\mathbf{S}_K \in \mathbb{R}^{m \times m}$ and the \textit{Knowledge Fusion Matrix} $\mathbf{M}_K \in \mathbb{R}^{m \times d}$ to clients for personalization.
This transmission results in a per-client downloading complexity of $\mathcal{O}(m^2 + md)$.





\newpage

\begin{table*}[t]
\centering
\caption{Performance comparisons across various baselines.
The best result is highlighted in \textbf{bold}, and the second-best result is \underline{underlined}.
The improvement refers to the performance advantage of our FedHiP scheme compared to the second-best result.
} 
\renewcommand{\arraystretch}{1.2}{
\resizebox{1.0\textwidth}{!}{
\begin{NiceTabular}{@{}l|cccccc|cccccc@{}}
\toprule
\multicolumn{1}{c|}{\multirow{4}{*}{{Baseline}}} &
  \multicolumn{6}{c|}{CIFAR-100~\cite{dataset_CIFAR}} &
  \multicolumn{6}{c}{ImageNet-R~\cite{dataset_ImageNet-R}} \\ \cmidrule(l){2-13} 
\multicolumn{1}{c|}{} &
  \multicolumn{3}{c|}{50 Clients} &
  \multicolumn{3}{c|}{100 Clients} &
  \multicolumn{3}{c|}{50 Clients} &
  \multicolumn{3}{c}{100 Clients} \\ \cmidrule(l){2-13} 
\multicolumn{1}{c|}{} &
  $\lambda=0.1$ &
  $\lambda=0.5$ &
  \multicolumn{1}{c|}{$\lambda=1.0$} &
    $\lambda=0.1$ &
  $\lambda=0.5$ &
  \multicolumn{1}{c|}{$\lambda=1.0$} &
    $\lambda=0.1$ &
  $\lambda=0.5$ &
  \multicolumn{1}{c|}{$\lambda=1.0$} &
    $\lambda=0.1$ &
  $\lambda=0.5$ &
  $\lambda=1.0$ \\ 
  \midrule
FedAvg~\cite{FedAvg} & 43.53\% & 47.75\% & \multicolumn{1}{c|}{48.91\%} & 39.28\% & 43.31\% & \multicolumn{1}{c|}{43.28\%} & 6.30\% & 7.14\% & \multicolumn{1}{c|}{7.61\%} & 4.90\% & 4.60\% & 5.00\% \\
FedAvg+FT~\cite{FedAvg} & 58.81\% & 41.38\% & \multicolumn{1}{c|}{37.73\%} & 49.05\% & 31.37\% & \multicolumn{1}{c|}{26.91\%} & 20.42\% & 8.14\% & \multicolumn{1}{c|}{6.15\%} & 18.68\% & 7.52\% & 4.84\% \\
FedProx~\cite{FedProx} & 42.74\% & 46.73\% & \multicolumn{1}{c|}{48.19\%} & 38.37\% & 42.38\% & \multicolumn{1}{c|}{42.54\%} & 6.28\% & 7.36\% & \multicolumn{1}{c|}{7.21\%} & 4.85\% & 4.50\% & 4.99\% \\
FedProx+FT~\cite{FedProx} & 57.63\% & 41.32\% & \multicolumn{1}{c|}{37.73\%} & 48.64\% & 31.38\% & \multicolumn{1}{c|}{26.91\%} & 20.40\% & 8.11\% & \multicolumn{1}{c|}{6.13\%} & 18.70\% & 7.53\% & 4.82\% \\
Ditto~\cite{Ditto} & 69.80\% & 50.18\% & \multicolumn{1}{c|}{43.72\%} & 66.80\% & 41.32\% & \multicolumn{1}{c|}{33.46\%} & 25.35\% & 8.57\% & \multicolumn{1}{c|}{6.05\%} & 22.85\% & 7.88\% & 5.10\% \\
FedALA~\cite{FedALA} & \underline{72.89\%} & \underline{52.60\%} & \multicolumn{1}{c|}{46.42\%} & \underline{70.63\%} & \underline{47.11\%} & \multicolumn{1}{c|}{38.20\%} & 32.01\% & 12.29\% & \multicolumn{1}{c|}{8.86\%} & \underline{28.20\%} & 10.15\% & 5.91\% \\
FedDBE~\cite{FedDBE} & 48.52\% & 49.96\% & \multicolumn{1}{c|}{\underline{50.11\%}} & 43.20\% & 44.58\% & \multicolumn{1}{c|}{\underline{44.41\%}} & 9.20\% & 8.46\% & \multicolumn{1}{c|}{8.31\%} & 7.40\% & 5.61\% & 5.78\% \\
FedAS~\cite{FedAS} & 56.70\% & 41.37\% & \multicolumn{1}{c|}{37.78\%} & 47.79\% & 30.73\% & \multicolumn{1}{c|}{27.28\%} & 20.32\% & 7.96\% & \multicolumn{1}{c|}{6.00\%} & 18.33\% & 7.37\% & 4.65\% \\
FedPCL~\cite{FedPCL} & 23.99\% & 11.33\% & \multicolumn{1}{c|}{12.52\%} & 25.83\% & 11.82\% & \multicolumn{1}{c|}{11.13\%} & 0.83\% & 0.88\% & \multicolumn{1}{c|}{1.13\%} & 1.72\% & 1.11\% & 1.23\% \\
FedSelect~\cite{FedSelect} & 72.46\% & 49.44\% & \multicolumn{1}{c|}{41.58\%} & 68.44\% & 42.71\% & \multicolumn{1}{c|}{34.41\%} & \underline{34.83\%} & \underline{14.94\%} & \multicolumn{1}{c|}{\underline{9.59\%}} & 28.11\% & \underline{11.01\%} & \underline{7.09\%} \\
\midrule
FedHiP (Ours) & \textbf{78.68\%} & \textbf{64.58\%} & \multicolumn{1}{c|}{\textbf{61.26\%}} & \textbf{76.92\%} & \textbf{63.23\%} & \multicolumn{1}{c|}{\textbf{60.08\%}} & \textbf{45.55\%} & \textbf{32.37\%} & \multicolumn{1}{c|}{\textbf{30.21\%}} & \textbf{43.77\%} & \textbf{31.55\%} & \textbf{28.06\%} \\
{{\textbf{Improvement $\uparrow$}}} & {{\textbf{5.79\%}}} & {\textbf{11.98\%}} & \multicolumn{1}{c|}{\textbf{11.15\%}} & {\textbf{6.29\%}} & {\textbf{16.12\%}} & \multicolumn{1}{c|}{\textbf{15.67\%}} & {\textbf{10.72\%}} & {\textbf{17.43\%}} & \multicolumn{1}{c|}{\textbf{20.62\%}} & {\textbf{15.57\%}} & {\textbf{20.54\%}} & {\textbf{20.97\%}} 
\\
\bottomrule
\end{NiceTabular}
}
}
\label{tab:overall}
\vspace{-0.22cm} 
\end{table*}

\begin{figure*}[!h]
  \centering
  \subfloat[Accuracy vs. Aggregation Round]{\includegraphics[width=0.6666\columnwidth]{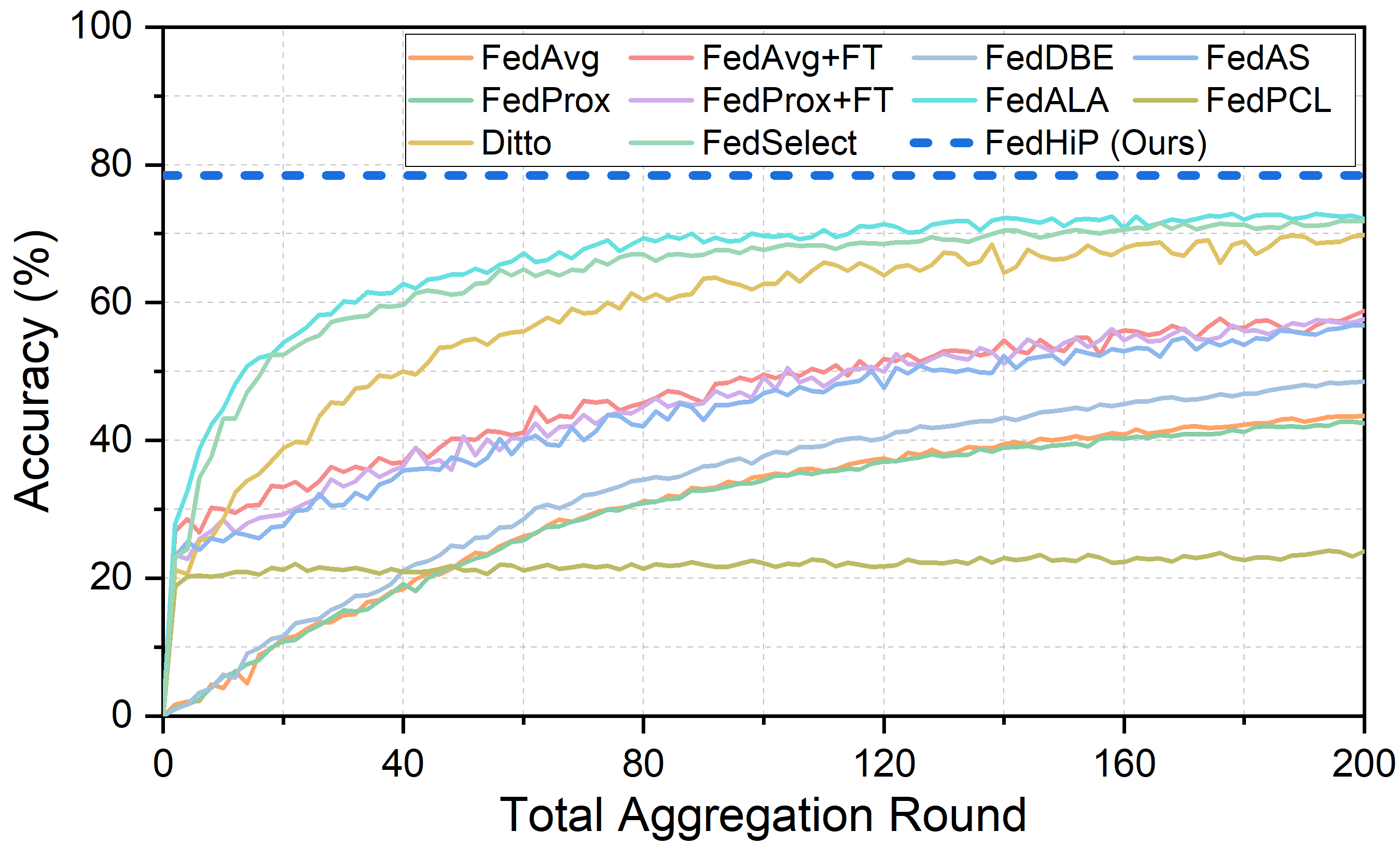}
  \label{fig:1a}}%
  \hfill
  \subfloat[Accuracy vs. Computation Overhead]{\includegraphics[width=0.6666\columnwidth]{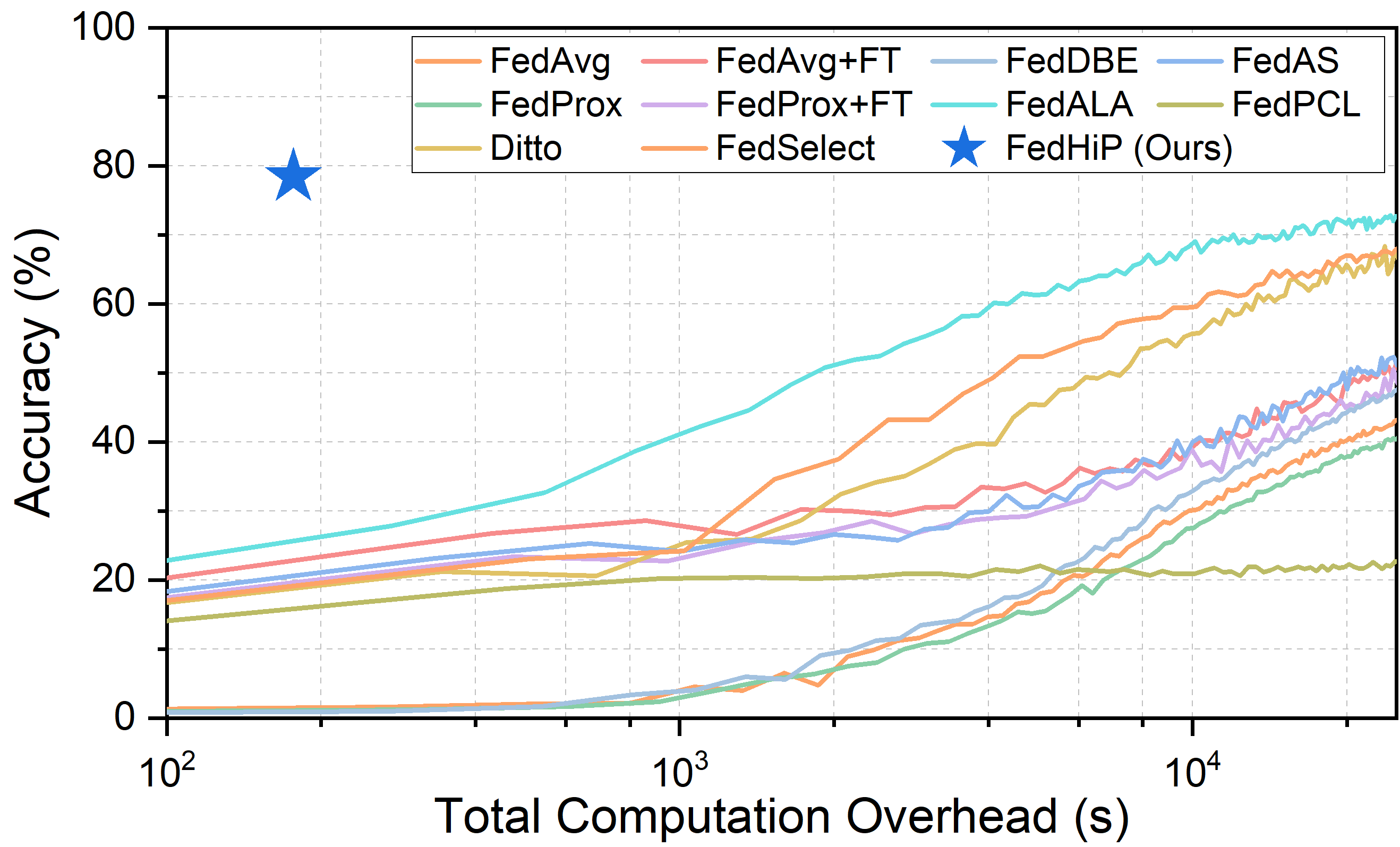}
  \label{fig:1b}}%
  \hfill
  \subfloat[Accuracy vs. Communication Overhead]{\includegraphics[width=0.6666\columnwidth]{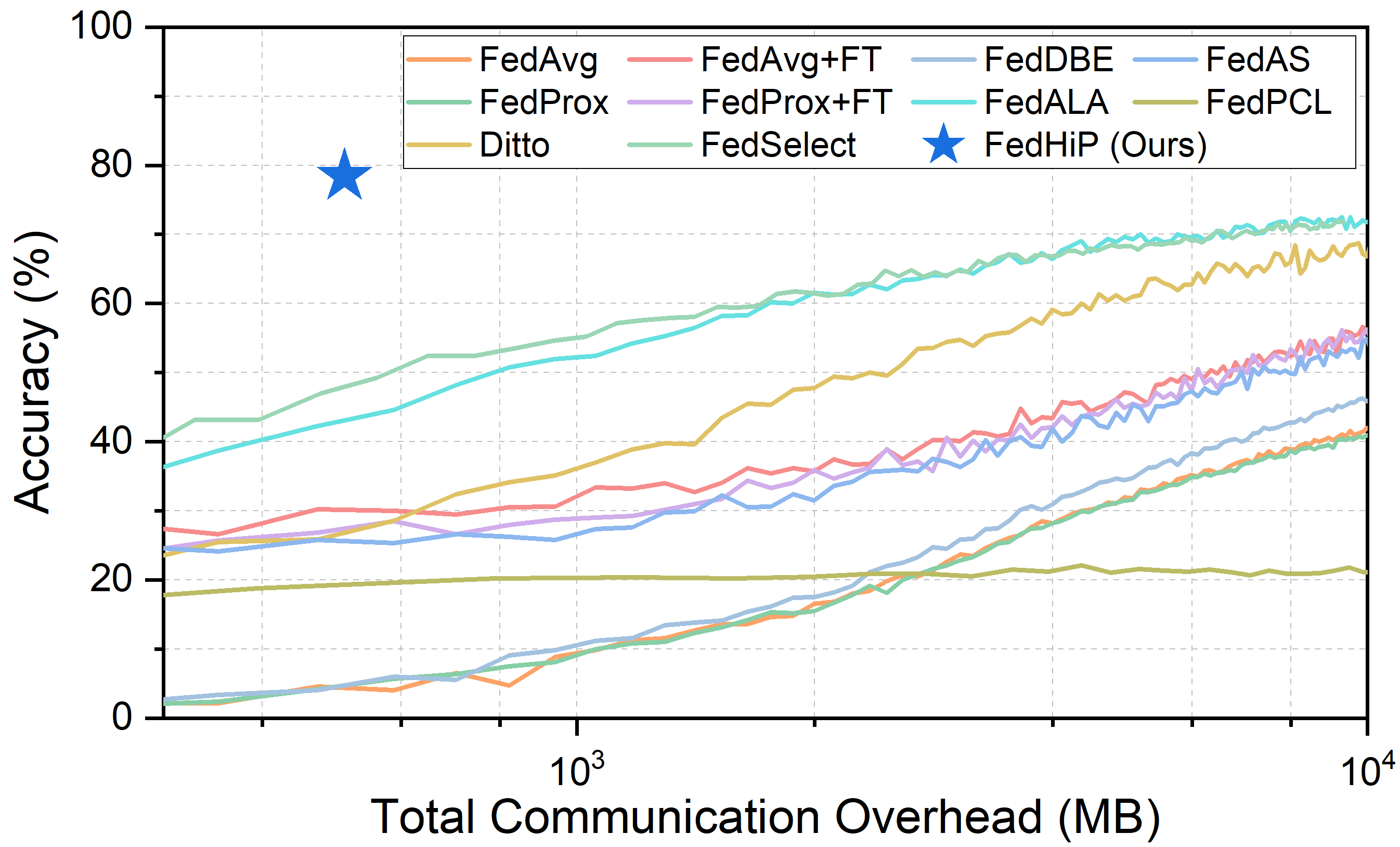}
  \label{fig:1c}}%
  \hfill
  \vspace{-0.12cm} 
  \caption{Efficiency evaluations on the CIFAR-100 dataset.}
  \label{fig:exp1}
    \vspace{-0.30cm} 
\end{figure*}

\begin{figure*}[!h]
  \centering
  \subfloat[Accuracy vs. Aggregation Round]{\includegraphics[width=0.6666\columnwidth]{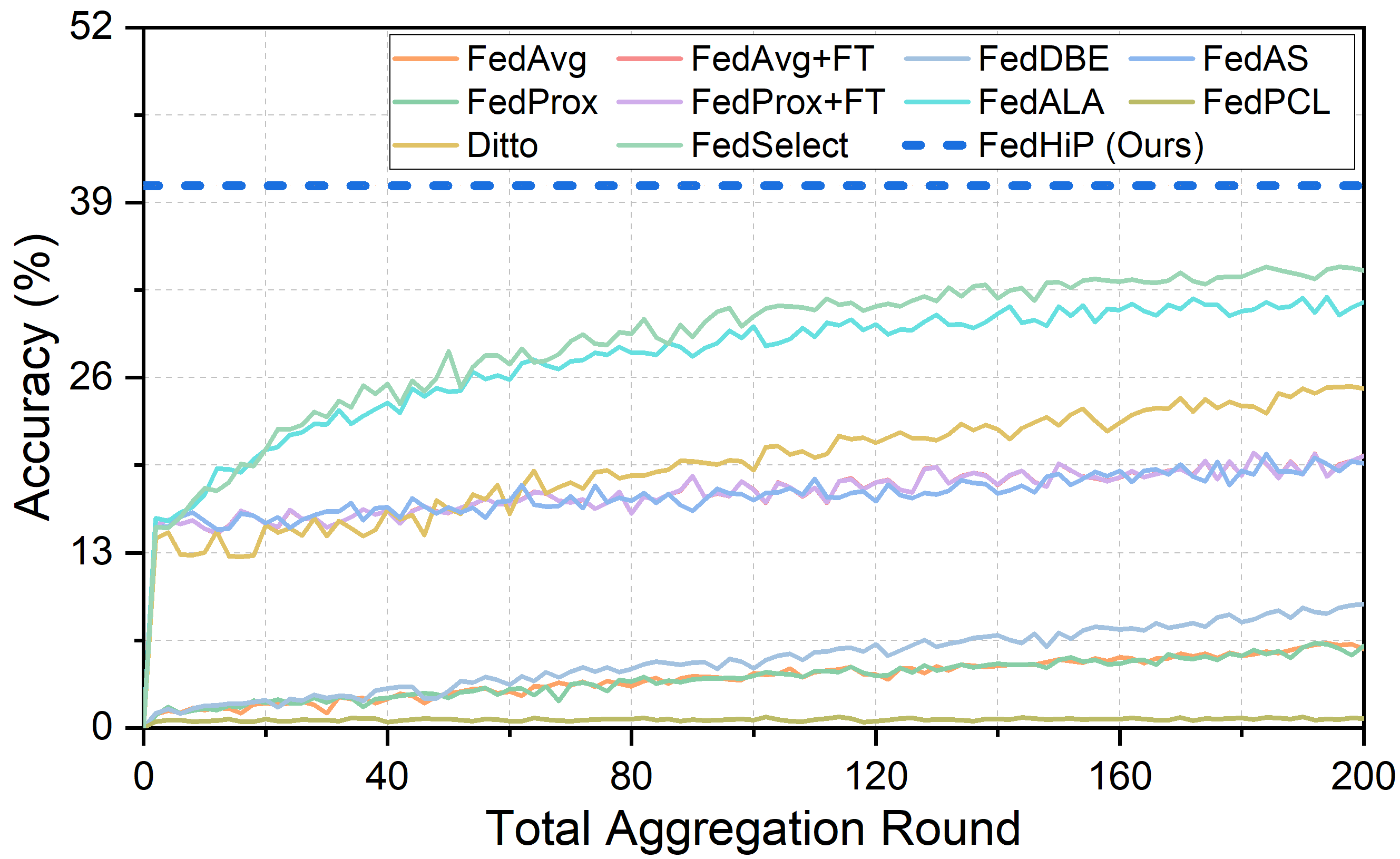}
  \label{fig:2a}}%
  \hfill
  \subfloat[Accuracy vs. Computation Overhead]{\includegraphics[width=0.6666\columnwidth]{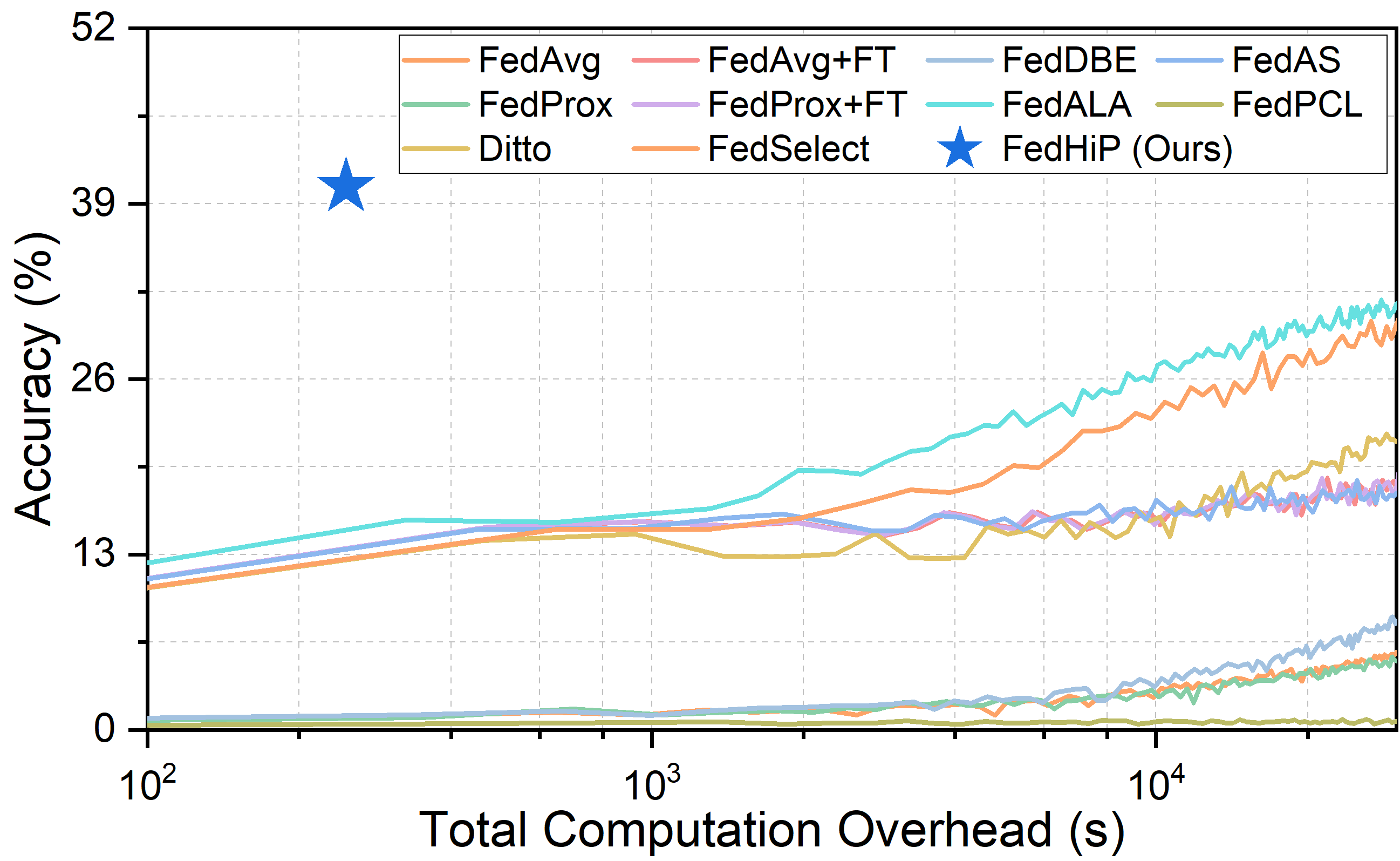}
  \label{fig:2b}}%
  \hfill
  \subfloat[Accuracy vs. Communication Overhead]{\includegraphics[width=0.6666\columnwidth]{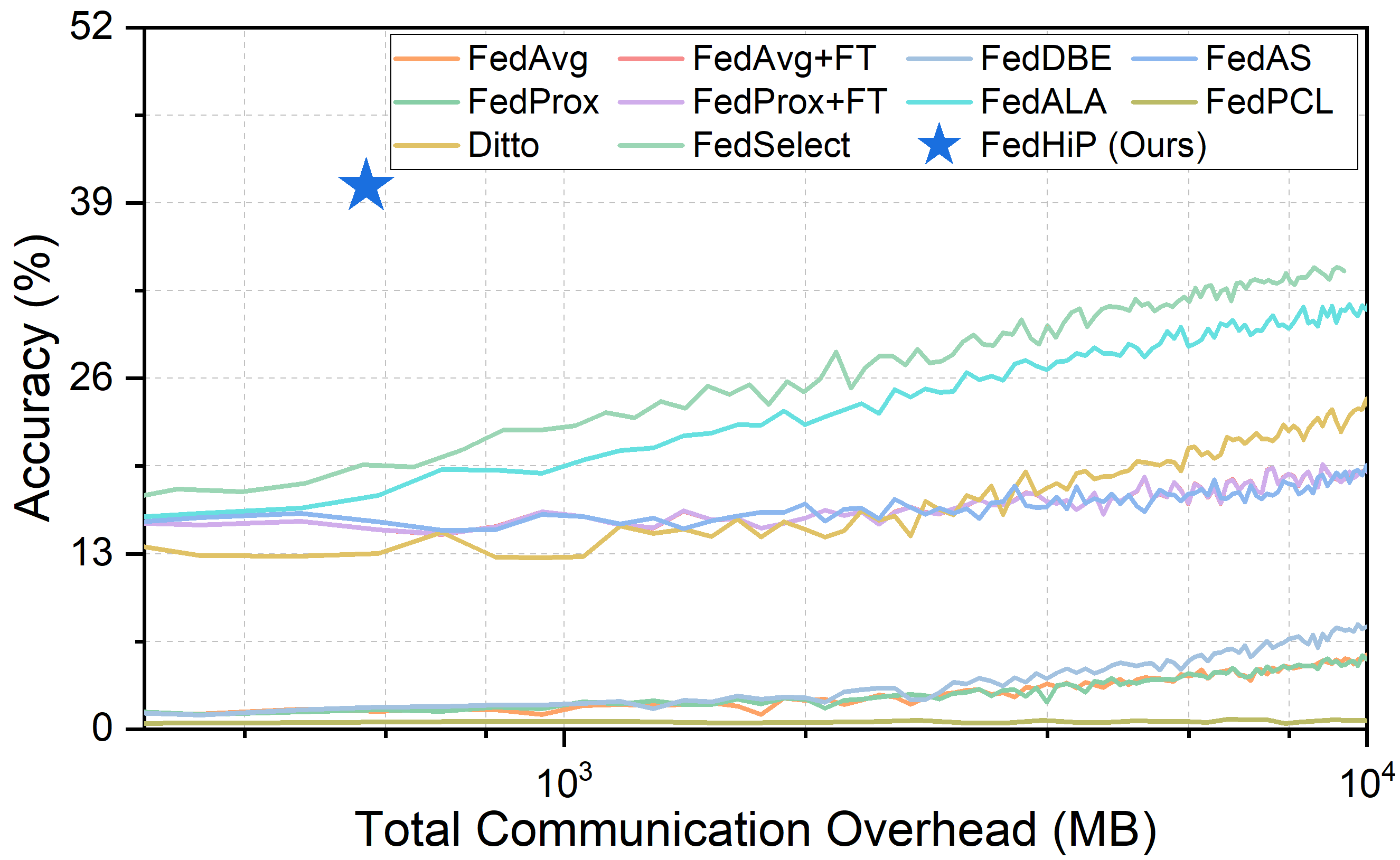}
  \label{fig:2c}}%
  \hfill
  \vspace{-0.12cm} 
  \caption{Efficiency evaluations on the ImageNet-R dataset.}
  \label{fig:exp2}
    \vspace{-0.47cm} 
\end{figure*}

\section{Experimental Evaluations}
\label{evaluation}

\subsection{Experimental Setup}

\subsubsection{Datasets and Settings}
To evaluate our FedHiP scheme's performance, we conduct our experiments on two challenging benchmark datasets: CIFAR-100~\cite{dataset_CIFAR} and ImageNet-R~\cite{dataset_ImageNet-R}.
We use the Dirichlet distribution with common concentration parameters $\lambda \in \{0.1, 0.5, 1.0\}$ to simulate varying degrees of non-IID environments.
Our experiments cover two settings with 50 and 100 clients.
To ensure experimental fairness, all methods utilize the same self-supervised pre-trained backbone, ViT-MAE-Base~\cite{ViT-MAE}, which can be sourced from a publicly accessible model.
Each client's dataset is partitioned into training and testing sets using an 8:2 split ratio.
All experiments are conducted on RTX 4090 GPUs using the PyTorch framework.

\subsubsection{Baselines and Metrics}
Here, we conduct comprehensive comparisons against diverse PFL baselines spanning multiple paradigms.
Traditional FL baselines include FedAvg~\cite{FedAvg} and FedProx~\cite{FedProx}, along with their fine-tuned variants (FedAvg+FT and FedProx+FT).
More recent PFL baselines comprise meta-learning method Ditto~\cite{Ditto}, knowledge-distillation method FedPCL~\cite{FedPCL}, and model-splitting methods including FedALA~\cite{FedALA}, FedDBE~\cite{FedDBE}, FedAS~\cite{FedAS}, and FedSelect~\cite{FedSelect}.
All baselines are based on gradients, using standard SGD for local training with a batch size of 32 and a learning rate of 0.005.
For each baseline, we employ 3 epochs for local training and 200 rounds for global communication across all experiments. 
As the core metric, we evaluate each method's final model on each client's local dataset and report the average accuracy.
Additionally, we also assess communication overhead through data transmission volume during uploading and downloading, and computational overhead through the total cost time across the entire method.


\subsection{Overall Comparisons}
TABLE~\ref{tab:overall} demonstrates the overall comparison results of our FedHiP scheme against various baselines. 
Overall, our FedHiP scheme consistently shows superior performance across a wide range of settings, including various datasets, different numbers of clients, and diverse levels of data heterogeneity.
Meanwhile, as $\lambda$ decreases, the heightened non-IID nature adversely affects the aggregation of global knowledge, leading to performance declines in traditional FL methods like FedAvg and FedProx.
Conversely, it also emphasizes personalization by altering local test datasets, thus improving the performance of PFL methods.

When the number of clients increases from 50 to 100, the baselines generally experience varying degrees of performance degradation.
In contrast, our FedHiP scheme maintains highly stable performance, as a benefit attributed to its heterogeneity-invariance property.
The slight performance fluctuations in our FedHiP scheme with varying client numbers are due to the randomness of the ViT-MAE backbone's patch sampling~\cite{ViT-MAE}.
Notably, despite our best efforts in hyperparameter tuning, the baselines consistently underperform on the very challenging ImageNet-R dataset.
Some even achieve below 10\% accuracy.
It highlights the limitations of gradient-based methods and, to some extent, explains why few PFL studies dare to conduct comparative analyses on such a challenging dataset in practice. 

\begin{table*}[t]
    \centering
    \renewcommand{\arraystretch}{1.2}
    \caption{Accuracy of our FedHiP scheme with various hyperparameter values of the term $\alpha$ on the CIFAR-100 dataset.}
    \label{table:sensitivity-1}
     \resizebox{1\textwidth}{!}{
        \begin{NiceTabular}{l c c c c c c c c c c c c c c c}
        \toprule
        \textbf{non-IID} & $\alpha = 0$ & $\alpha = 5$ & $\alpha = 10$ & $\alpha = 15$ & $\alpha = 20$ & $\alpha = 25$ & $\alpha = 30$ & $\alpha = 35$ & $\alpha = 40$ & $\alpha = 45$ & $\alpha = 50$ & $\alpha = 55$ & $\alpha = 60$\\
        \midrule
        $\lambda=0.1$ & 57.64\% & 73.48\% & 76.83\% & 78.1\% & 78.51\% & \textbf{78.68\%} & 78.65\% & 78.53\% & 78.47\% & 78.45\% & 78.31\% & 78.10\% & 77.97\% \\
        $\lambda=0.5$  & 57.32\% & 62.77\% & 64.36\% & \textbf{64.58\%} & 64.39\% & 63.96\% & 63.6\% & 63.16\% & 62.79\% & 62.29\% & 61.9\% & 61.43\% & 61.09\% 
        \\
        $\lambda=1.0$ & 57.58\%  & 60.12\% & 60.97\% & \textbf{61.26\%} & 60.94\% & 60.53\% & 59.78\% & 59.32\% & 58.45\% & 58.04\% & 57.88\% & 57.33\% & 56.94\% \\  
        \bottomrule
        \end{NiceTabular}
        }
    \vspace{0.1cm}
\end{table*}
\begin{table*}[t]
    \centering
    \renewcommand{\arraystretch}{1.2}
    \caption{Accuracy of our FedHiP scheme with various hyperparameter values of the term $\beta$ on the CIFAR-100 dataset.}
    \label{table:sensitivity-2}
     \resizebox{1\textwidth}{!}{
        \begin{NiceTabular}{ l c c c c c c c c c c c c c c}
        \toprule
        \textbf{non-IID} & $\beta = 0$ & $\beta = 5$ & $\beta = 10$ & $\beta = 15$ & $\beta = 20$ & $\beta = 25$ & $\beta = 30$ & $\beta = 35$ & $\beta = 40$ & $\beta = 45$ & $\beta = 50$ & $\beta = 55$ & $\beta = 60$ \\
        \midrule
        $\lambda=0.1$ & 78.51\% & \textbf{78.65\%} & 78.51\% & 78.53\% & 78.54\% & 78.51\% & 78.49\% & 78.37\% & 78.32\% & 78.23\% & 78.28\% & 78.22\% & 78.17\% \\
        $\lambda=0.5$ & 64.39\% & \textbf{64.43\%} & 64.30\% & 64.27\% & 64.17\% & 64.06\% & 63.95\% & 63.85\% & 63.77\% & 63.78\% & 63.75\% & 63.66\% & 63.60\% \\
        $\lambda=1.0$ & \textbf{60.94\%} & 60.88\% & 60.77\% & 60.54\% & 60.49\% & 60.38\% & 60.27\% & 60.17\% & 60.04\% & 59.96\% & 59.90\% & 59.86\% & 59.74\% \\
        \bottomrule
        \end{NiceTabular}
        }
    \vspace{0.084cm}
\end{table*}

\subsection{Efficiency Evaluations}
Subsequently, we conduct comprehensive evaluations of our FedHiP scheme against various baselines.
We demonstrate the results on the CIFAR-100 dataset and ImageNet-R dataset in Fig.~\ref{fig:exp1} and Fig.~\ref{fig:exp2}, respectively.
In the efficiency evaluations, we fix the number of clients at 50 and set the heterogeneity parameter $\lambda=0.1$ for all methods.
The evaluation results for both datasets include three subfigures. 
The subfigure (a) shows the convergence behavior of different baselines as a function of aggregation rounds.
The subfigures (b) and (c) demonstrate the corresponding computation and communication overhead for different methods to reach various accuracy values.


Let's first analyze the evaluation results in Fig.~\ref{fig:exp1}.
As shown in its subfigure (a), all baselines rely on multiple aggregation rounds to converge, while our FedHiP scheme only requires a single round of aggregation from each client.
Furthermore, our FedHiP scheme reaches an accuracy that surpasses the final convergence accuracy of all baselines, thereby achieving state-of-the-art performance.
Then, subfigures (b) and (c) more intuitively show our efficiency advantages in terms of computation and communication.
Specifically, our FedHiP scheme not only achieves an accuracy improvement of over 5\% compared to the baselines but also reduces computation and communication overheads by as much as 80\% and 95\%, respectively.



In Fig.~\ref{fig:exp2}, the obtained evaluation results are consistent overall with those shown in Fig.~\ref{fig:exp1}.
However, it's worth noting that during evaluations on the ImageNet-R dataset, many baselines exhibited non-convergence with very low accuracy, despite our best efforts in hyperparameter tuning. This evidence, on one hand, indicates that many PFL baselines struggle when applied to such a challenging dataset.
On the other hand, this evidence further highlights the advantages of our FedHiP scheme on the challenging dataset, as it can achieve absolute convergence and obtain state-of-the-art performance in a gradient-free manner.


\subsection{Sensitivity Analyses}

Last but not least, we further conducted detailed sensitivity analyses for our FedHiP scheme.
Specifically, according to our proposed optimization in~\eqref{eq:person-1}, our FedHiP scheme has only two hyperparameters: $\alpha$ to control the level of personalization, and $\beta$ to control the degree of regularization penalty.
The very few hyperparameters also free us from the hassle of tedious hyperparameter tuning when applying our scheme in practice.

Here, by fixing 50 clients and keeping the hyperparameter $\beta=0$, we adjust $\alpha$ to different values to observe its impact on the performance of our scheme, as presented in TABLE~\ref{table:sensitivity-1}.
The results indicate that the personalization term in~\eqref{eq:person-1} significantly impacts the performance of our scheme, and both excessively small or large values of $\alpha$ noticeably degrade its performance.
The optimal value of $\alpha$ is generally around 10 to 30, and this value is influenced by the heterogeneity parameter $\lambda$.
Notably, when considering the special case of $\alpha=0$, it can be viewed as an ablation study to validate the efficacy of the proposed phase 3 within our FedHiP scheme. 
This personalization phase can improve FedHiP's performance by up to {\color{black}{3.68\%-21.08\%.}}

Then, by fixing 50 clients and setting the hyperparameter $\alpha=20$, we adjust $\beta$ to different values to observe its impact on the performance of our scheme, as shown in TABLE~\ref{table:sensitivity-2}.
The results indicate that on the CIFAR-100 dataset, our FedHiP scheme appears to be not sensitive to the hyperparameter $\beta$.
This evidence suggests that the ridge regularization term in~\eqref{eq:person-1} has a minor impact on the overall performance of our scheme.
It brings good news: we can easily achieve good performance without laboriously tuning the hyperparameter $\beta$.


%


\section{Conclusion and Discussion}
\label{conclusion}
\subsection{Conclusion}
Identifying the inherent sensitivity of gradient-based updates to non-IID data, in this paper, we propose our FedHiP scheme to fundamentally address this issue by avoiding gradient-based updates via analytical solutions in PFL.
Specifically, we introduce a foundation model as a frozen backbone for gradient-free feature extraction, and develop an analytic classifier after the backbone for gradient-free training.
To achieve both collective generalization and individual personalization, we design a three-phase analytic framework for our FedHiP scheme.
Both theoretical analyses and extensive experiments validate the superiority of our FedHiP scheme.
To the best of our knowledge, we are the first to introduce the concept of analytic learning into PFL. 
Consequently, our FedHiP scheme also represents the pioneer to achieve the ideal property of heterogeneity invariance within the PFL domain.

\subsection{Discussion}
The principal limitation of our FedHiP scheme is its use of a frozen foundation model for feature extraction.
However, in the current landscape dominated by large AI models, this design is not overly restrictive, and many works have adopted similar approaches in PFL. 
In particular, the foundation model can be readily downloaded from an open-source repository, or trained centrally on public datasets.
Nevertherless, this minor limitation motivates us to further explore an adjustable backbone within our FedHiP scheme in the future, aiming to enhance the capability of personalized feature extraction.

Another limitation of our FedHiP scheme is that the single-layer analytic classifier solely captures the linear relationship between the features and the labels.
Despite this limitation, our FedHiP scheme already achieves state-of-the-art performance due to its ideal property of heterogeneity invariance, aided by a powerful feature extractor.
Besides, our FedHiP scheme holds promising potential for further enhancing its nonlinear classification capability by incorporating classic machine learning techniques, such as kernel methods and ensemble learning.
In the future, we will also focus on extending our FedHiP scheme to include multi-layer analytic classifier with better capability.




\clearpage

\bibliographystyle{IEEEtran}
\bibliography{main}

\clearpage

\begin{IEEEbiography}
[{\includegraphics[width=0.98in,height=1.2in,clip,keepaspectratio]{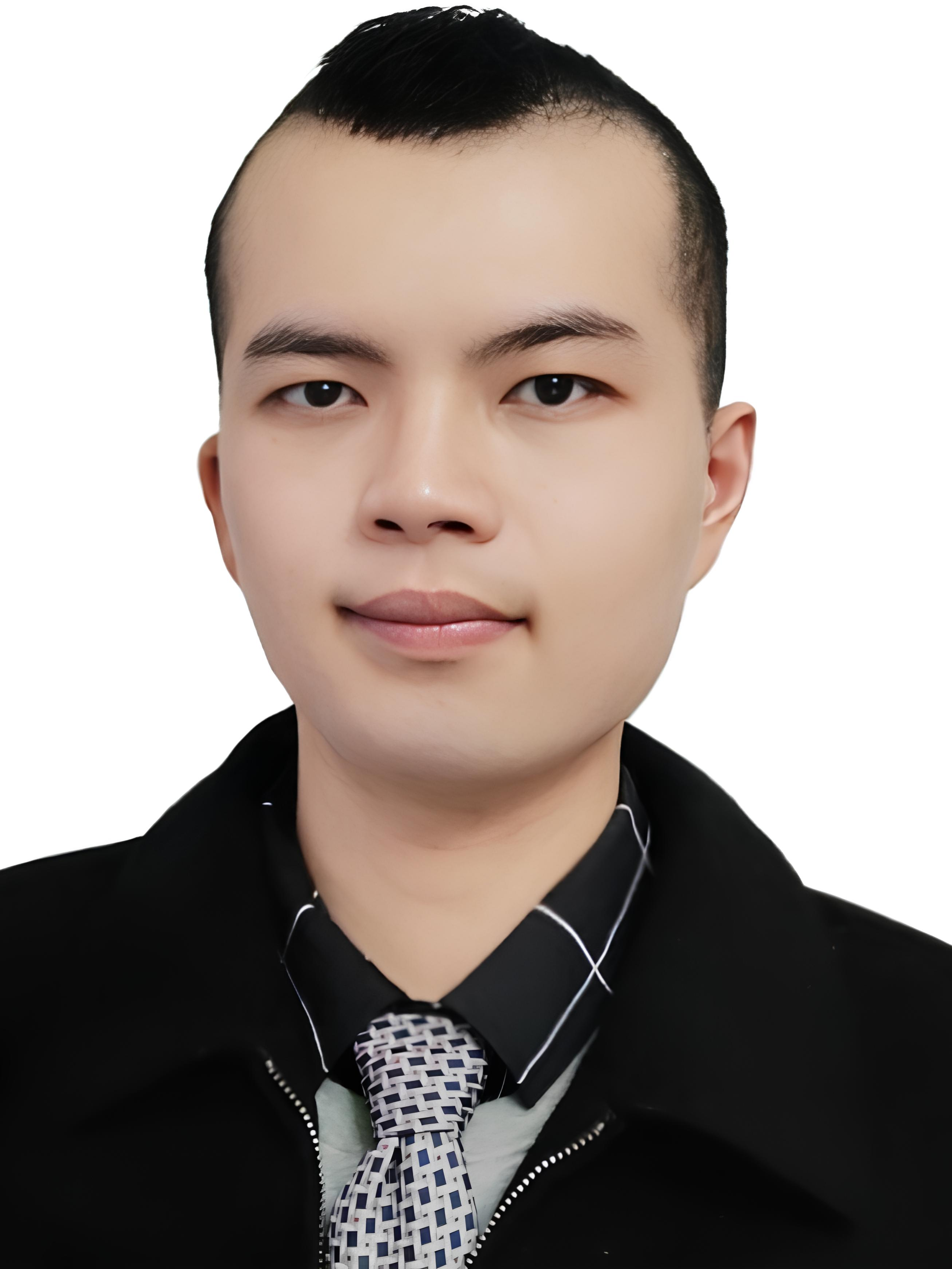}}]{Jianheng Tang}
is currently pursuing a Ph.D. degree at the School of Computer Science, Peking University, China. 
His research interests include machine learning, Internet of Things, and mobile computing.
\end{IEEEbiography}
\vspace{-0.5cm}

\begin{IEEEbiography}
[{\includegraphics[width=0.98in,height=1.2in,clip,keepaspectratio]{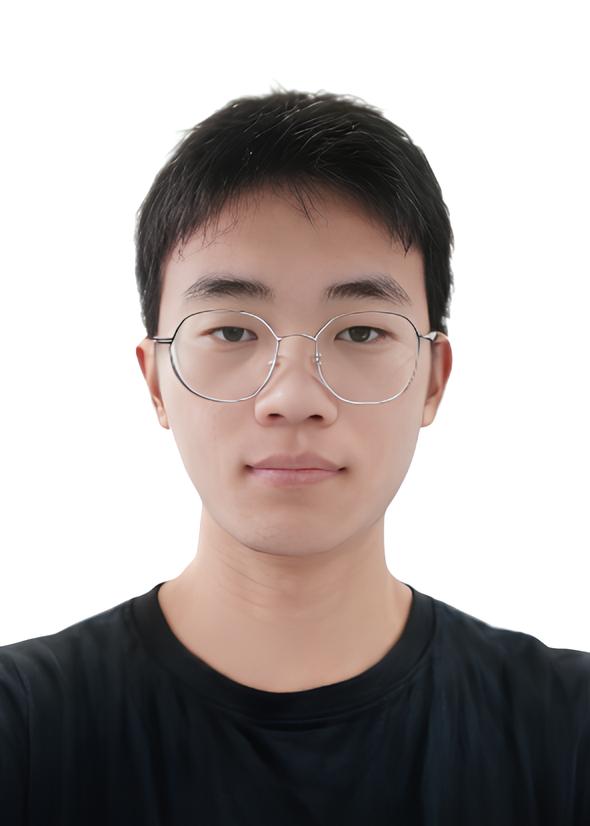}}]{Zhuirui Yang}
is currently a student at the School of Computer Science and Engineering, Central South University, China.
His major research interests include federated learning and Internet of Things.
\end{IEEEbiography}
\vspace{-0.5cm}

\begin{IEEEbiography}
[{\includegraphics[width=0.98in,height=1.2in,clip,keepaspectratio]{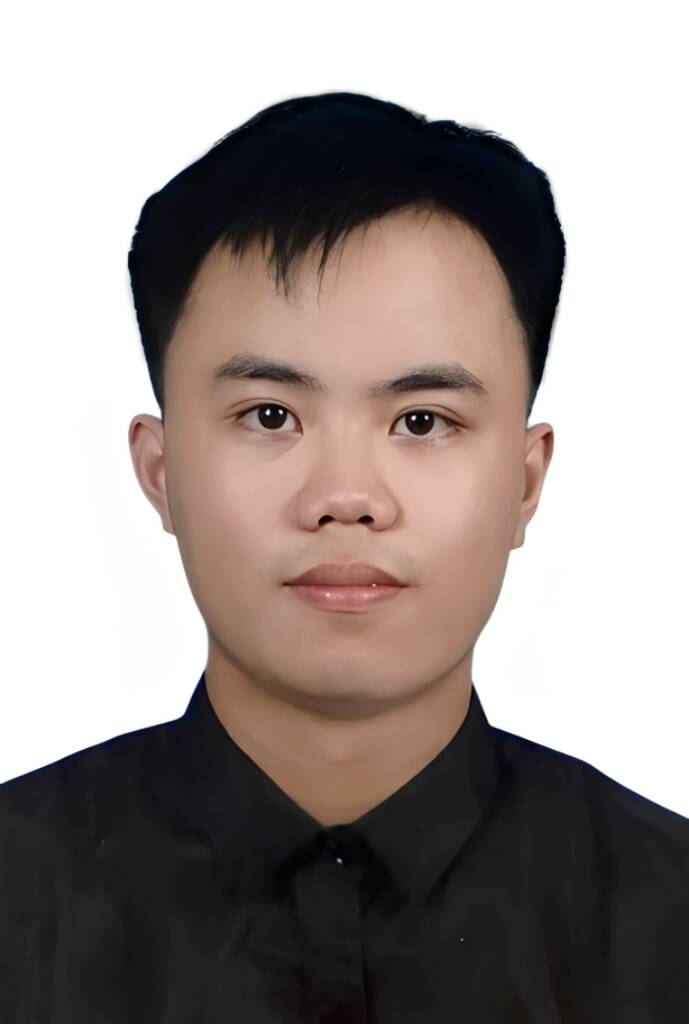}}]{Jingchao Wang}
is currently pursuing a Ph.D. degree at the School of Computer Science, Peking University, China. 
His research interests include artificial intelligence and machine learning.
\end{IEEEbiography}
\vspace{-0.5cm}

\begin{IEEEbiography}
[{\includegraphics[width=0.98in,height=1.2in,clip,keepaspectratio]{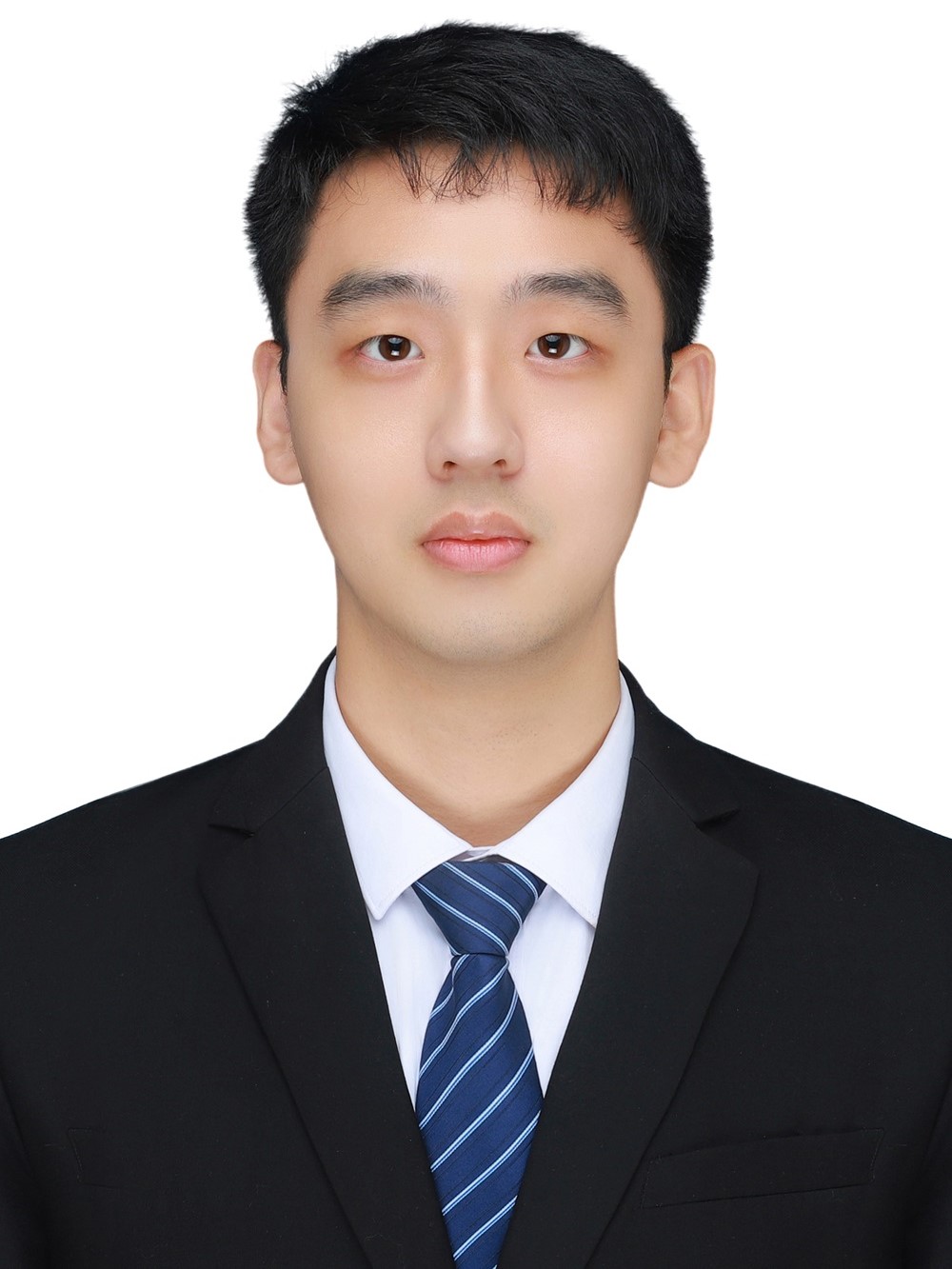}}]{Kejia Fan}
is currently pursuing his Master's degree at the School of Electronic Information, Central South University, China. 
His research interests include mobile crowd sensing and Internet of Things. 
\end{IEEEbiography}
\vspace{-0.5cm}

\begin{IEEEbiography}
[{\includegraphics[width=0.98in,height=1.2in,clip,keepaspectratio]{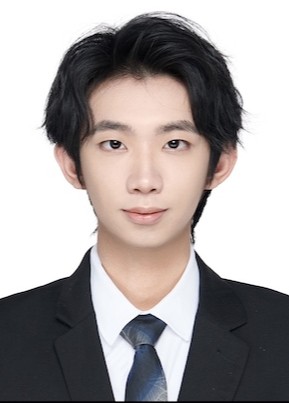}}]{Jinfeng Xu}
is currently a PhD Student of the Department of Electrical and Electronic Engineering at the University of Hong Kong.
His research interests include recommendation systems, data privacy, self-supervised learning, and federated learning.
\end{IEEEbiography}
\vspace{-0.5cm}

\begin{IEEEbiography}
[{\includegraphics[width=0.98in,height=1.2in,clip,keepaspectratio]{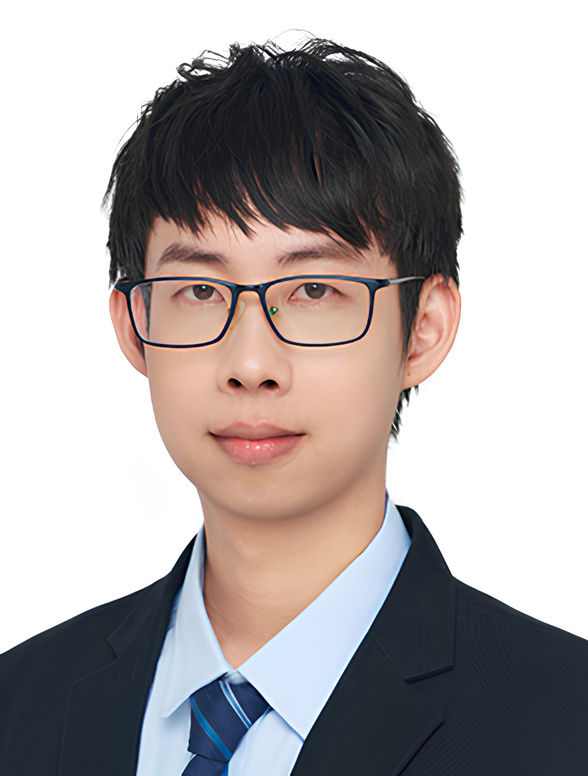}}]{Huiping Zhuang}\textit{\textbf{(Member, IEEE)}}
received the B.S. and M.E. degrees from South China University of Technology, Guangzhou, China, in 2014 and 2017, respectively, and the Ph.D. degree from Nanyang Technological University, Singapore, in 2021.
He is currently an Associate Professor with the Shien-Ming Wu School of Intelligent Engineering, South China University of Technology. 
His research interests include signal processing and machine learning.
\end{IEEEbiography}
\vspace{-0.5cm}

\begin{IEEEbiography}
[{\includegraphics[width=0.98in,height=1.2in,clip,keepaspectratio]{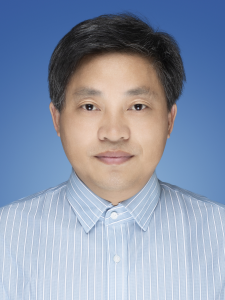}}]{Anfeng Liu}
received the M.Sc. and Ph.D. degrees from Central South University, China, in 2002 and 2005, respectively, both in computer science. 
He was a visiting scholar with the Broadband Communications Research (BBCR) Lab at the University of Waterloo in Canada from 2011 to 2012.
He is currently a professor at the School of Electronic Information, Central South University, China.
His major research interests include wireless sensor networks, Internet of Things, and mobile computing. 
Dr. Liu has published 4 books and over 300 international journal and conference papers with over 18,000 citations, among which there are more than 30 ESI highly-cited papers. 
Most of his works were published in premium conferences and journals, including IEEE JSAC, IEEE TMC, IEEE TPDS, IEEE TDSC, IEEE TIFS, IEEE TSC, IEEE TWC, etc.
His research has been supported by the National Basic Research Program of China
(973 Program) and the National Natural Science Foundation of China for six times.
\end{IEEEbiography}
\vspace{-0.5cm}

\begin{IEEEbiography}
[{\includegraphics[width=0.98in,height=1.2in,clip,keepaspectratio]{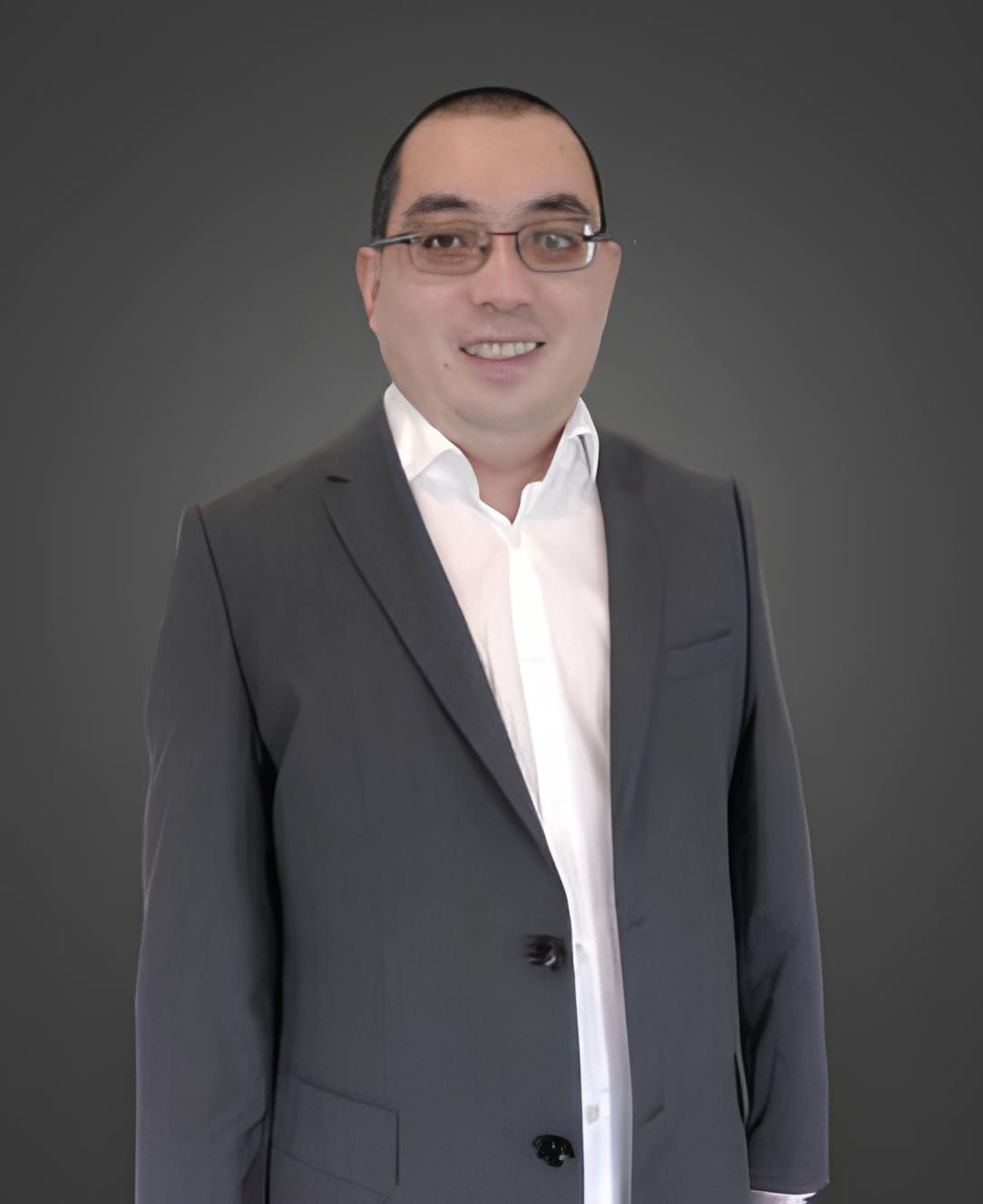}}]{Houbing Song}\textit{\textbf{(Fellow, IEEE)}}
received the Ph.D. degree in electrical engineering from the University of Virginia, Charlottesville, VA, in August 2012.
He is currently a Tenured Associate Professor, the Director of the NSF Center for Aviation Big Data Analytics (Planning), and the Associate Director for Leadership of the DoT Transportation Cybersecurity Center for Advanced Research and Education,
University of Maryland, Baltimore County (UMBC), Baltimore, MD.
He has served as one of the Co-Editors-in-Chief for IEEE TII, a Guest Editor for IEEE JSAC, and an Associate Editor for IEEE TAI, IEEE IOTJ, and IEEE TITS, etc.
He is an IEEE Fellow, an ACM Distinguished Member, an ACM Distinguished Speaker, and an IEEE Vehicular Technology Society Distinguished Lecturer.
He has been a Highly Cited Researcher identified by Clarivate, a Top 1000 Computer Scientist identified by Research.com, and an IEEE Impact Creator since 2023.
He received the IEEE Harry Rowe Mimno Award and 10+ Best Paper Awards from major international conferences. 
His research interests include cyber-physical systems and Internet of Things.
\end{IEEEbiography}
\vspace{-0.5cm}

\begin{IEEEbiography}
[{\includegraphics[width=0.98in,height=1.2in,clip,keepaspectratio]{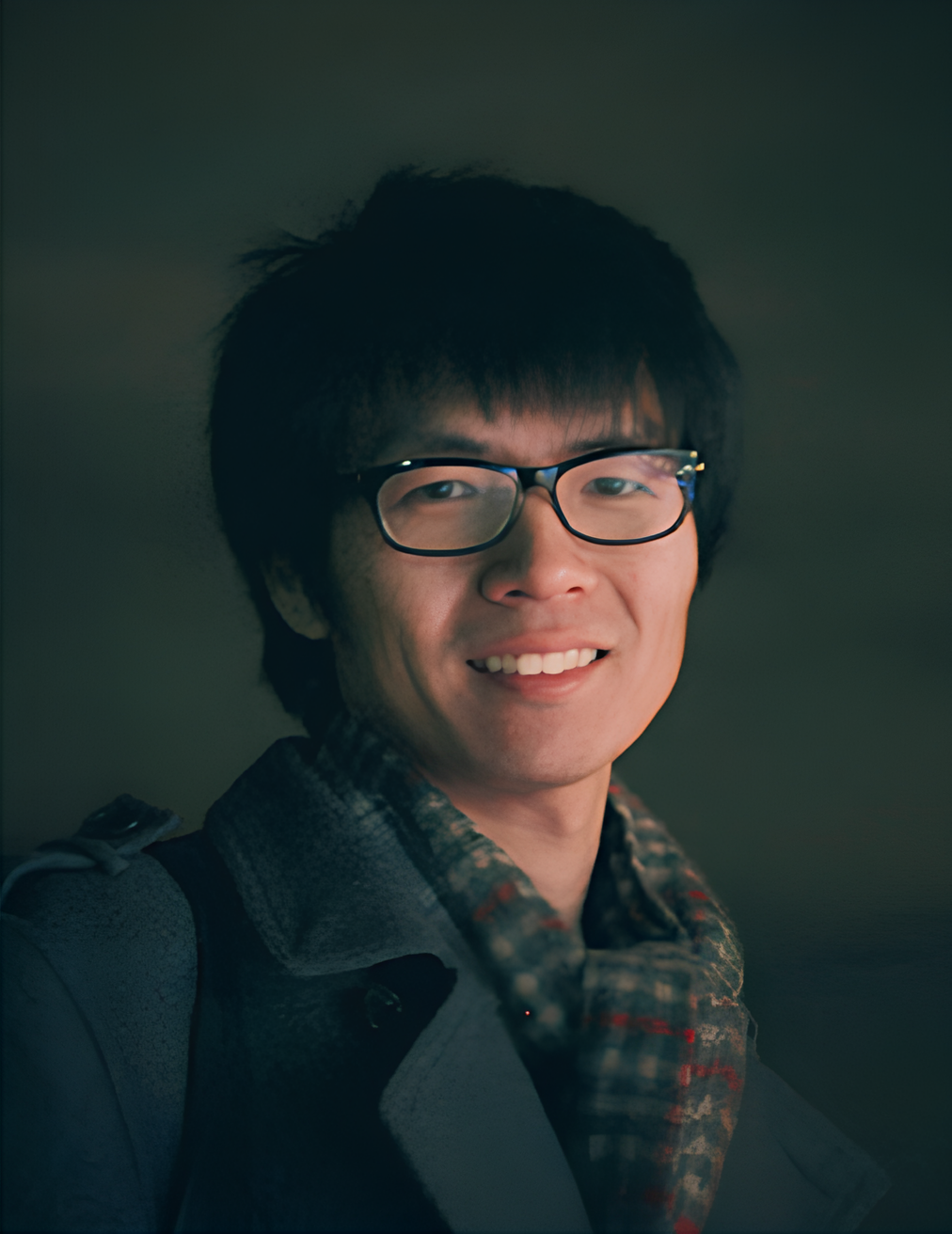}}]{Leye Wang}\textit{\textbf{(Member, IEEE)}}
received the PhD degree in computer science from TELECOM SudParis and University Paris 6, France, in 2016.
He is currently a tenured associate professor with the Key Lab of High Confidence Software Technologies, Peking University, Ministry of Education (MOE), and with the School of Computer Science, Peking University, China. He was a postdoctoral researcher with the Hong Kong University of Science and Technology. His research interests include ubiquitous computing, mobile crowdsensing, and urban computing.
\end{IEEEbiography}
\vspace{-0.5cm}

\begin{IEEEbiography}
[{\includegraphics[width=0.98in,height=1.2in,clip,keepaspectratio]{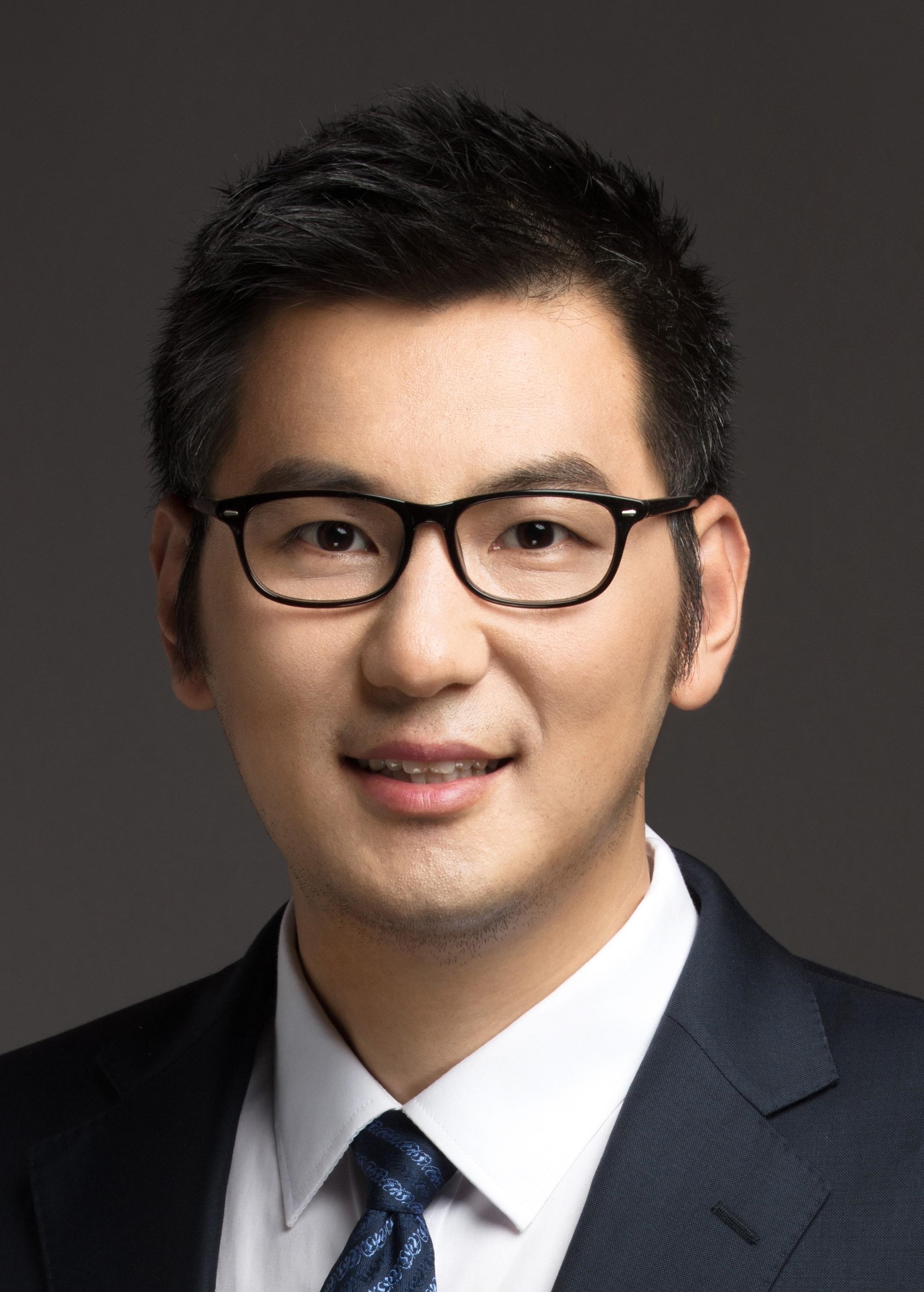}}]{Yunhuai Liu}\textit{\textbf{(Member, IEEE)}}
received his B.E degree in Computer Science and Technology from Tsinghua University, Beijing in 2000, and PhD degree in Computer Science and Engineering from Hong Kong University of Science and Technology in 2008.
He is currently a Full Professor at the School of Computer Science, Peking University, China.
He was a recipient of the National Talented Young Scholar Program (2013), and the National Distinguished Young Scholar Program (2019), funded by the National Science Foundation China (NSFC).
He has been honored with Best Paper Awards from ACL, IEEE ICDCS, and IEEE SANER.
His research interests include wireless networks, mobile computing, cyber-physical systems, and Internet of Things. 
Dr. Liu has served as an Associate Editor for prestigious journals such as IEEE TPDS, IEEE TNSE, and has been a TPC member for leading conferences including ACM Sensys and IEEE INFOCOM.
He is currently the Vice Chair of ACM China Council.
\end{IEEEbiography}

\end{document}